\documentclass[twoside]{article}

\usepackage[preprint]{aistats2027}

\usepackage{amssymb}
\usepackage{amsmath}
\usepackage{amsthm}
\usepackage{booktabs}
\usepackage{subcaption}
\usepackage{float}
\usepackage{tikz}
\usetikzlibrary{arrows.meta,backgrounds,calc,fit}
\usepackage{hyperref}
\usepackage{url}
\usepackage{enumitem}
\usepackage{graphicx} 

\usepackage[round]{natbib}

\newtheorem{theorem}{Theorem}[section]
\newtheorem{proposition}[theorem]{Proposition}
\newtheorem{lemma}[theorem]{Lemma}
\newtheorem{corollary}[theorem]{Corollary}
\newtheorem{assumption}[theorem]{Assumption}
\newtheorem{problem}[theorem]{Problem}
\allowdisplaybreaks
\floatstyle{ruled}
\newfloat{algorithm}{tbp}{loa}
\floatname{algorithm}{Algorithm}
\floatstyle{plain}
\hypersetup{
	colorlinks=true,
	linkcolor=black,
	filecolor=blue,
	citecolor = blue,      
	urlcolor=blue,
}

\newcommand{\Ind}{\mathbf{1}}
\newcommand{\TV}{\operatorname{TV}}

\DeclareMathOperator*{\argmin}{arg\,min}
\DeclareMathOperator*{\argmax}{arg\,max}
\begin{document}

%
\runningtitle{Reinforcement Learning with Conformal Action Sets}

%

\twocolumn[

\aistatstitle{Reinforcement Learning with Conformal Action Sets:\\An Application to Sequential Recommendation}

\aistatsauthor{ Wenwen Si \And Honghao Wei  }

\aistatsaddress{ University of Pennsylvania \And  Washington State University  }  ]

\begingroup
\hbadness=1400
\begin{abstract}
 Sequential recommenders typically use a fixed slate size even though the number of useful alternatives changes within a session. We propose Reinforcement Learning with Calibrated Pruning (RLCP), which adapts the retained action set using critic scores and an online threshold. The threshold is updated from binary feedback indicating whether the set contains an action in a proxy target. We prove a deterministic bound on the observed proxy miss rate along adaptive trajectories. To quantify the effect of pruning on reward, we derive an exact decomposition of value loss into filtering and selection losses. Under explicit proxy and critic approximation conditions, this decomposition yields a finite session reward bound that also accounts for imperfect selection and set truncation, without requiring the learning parameters to converge. Experiments on KuaiRand-Pure and MovieLens 1M compare two RLCP implementations with four RL baselines. In each of the 19 configurations, at least one RLCP variant achieves the highest catalog diversity, reaching $1.11\times$ to $5.21\times$ that of the strongest baseline, with competitive session depth and no larger retained sets.
\end{abstract}
\endgroup
\vspace{-2ex}
\section{Introduction}

Sequential recommendation aims to improve system utility over an entire user session, rather than optimize each interaction in isolation in a long-term deployment. Since a recommendation can affect both the current response and the user's subsequent interests, reinforcement learning (RL) provides a natural framework for this objective~\citep{afsar2022reinforcement,chen2021survey}. Within this framework, many RL recommenders present a slate with a \textit{fixed predetermined number} of items at every request~\citep{zhao2018deep,ie2019reinforcement,zhao2023kuaisim}. However, a fixed cardinality does not reflect how many useful alternatives are available and cannot account for changes in their quality. For example, a small set may suffice when preferences are clear, whereas a broader set can be valuable when several alternatives remain plausible. This mismatch motivates a different design criterion: retain a compact set while controlling the risk of excluding all valuable actions.

Achieving this objective requires more than replacing the top $M$ ranking rule with a threshold. First, the optimal long-term values used to define valuable actions are unknown. Although a learned critic can rank items, its score gaps are not themselves calibrated measures of retention. Second, filtering changes the available actions and, consequently, the states visited later in the session. As a result, a threshold calibrated under one policy need not satisfy a population retention requirement under another. Third, even when a valuable item is retained, the downstream actor or user choice mechanism may consistently select a poor alternative. Retention and selection must therefore be distinguished: keeping a valuable item does not by itself ensure a good decision. Together, these challenges call for an analysis that connects action set construction to sequential decision quality.

To address these challenges, we introduce \emph{Reinforcement Learning with Calibrated Pruning} (RLCP). In RLCP, a critic scores candidate items, an online threshold determines which items are retained, and a downstream selector chooses from the retained set. To calibrate the threshold, we update it using a binary signal indicating whether the raw set contains an action in a proxy target. Following adaptive conformal inference~\citep{gibbs2021adaptive,gibbs2024online,angelopoulos2024online}, this update gives a deterministic bound on the observed proxy miss frequency even as the critic, policy, and visited states change. The guarantee therefore concerns observable proxy retention, rather than retention defined by the unknown optimal value function. 

To account for how pruning affects future decisions, we evaluate expected set size under the state distribution of the policy that uses the retained actions. We prove that this policy's value loss equals the discounted accumulation of two terms: the loss caused by filtering and the loss incurred by selecting among the retained actions. Under a condition relating proxy scores to optimal action values, calibration bounds the filtering loss, while critic accuracy and the downstream selection rule bound the selection loss. Combining these bounds gives a finite session reward guarantee, including the additional loss from imposing a maximum set size. The result allows the critic, policy, and threshold to change throughout the session and does not require their parameters to converge. Experiments in 19 simulator configurations show that RLCP can reach a broader item catalog while maintaining competitive session depth under the same set size cap as the baselines, as summarized in Figure~\ref{fig:overall_comparison}. The main contributions are summarized below.

\begin{figure}[t]
   \centering
   \includegraphics[height=5.2cm]{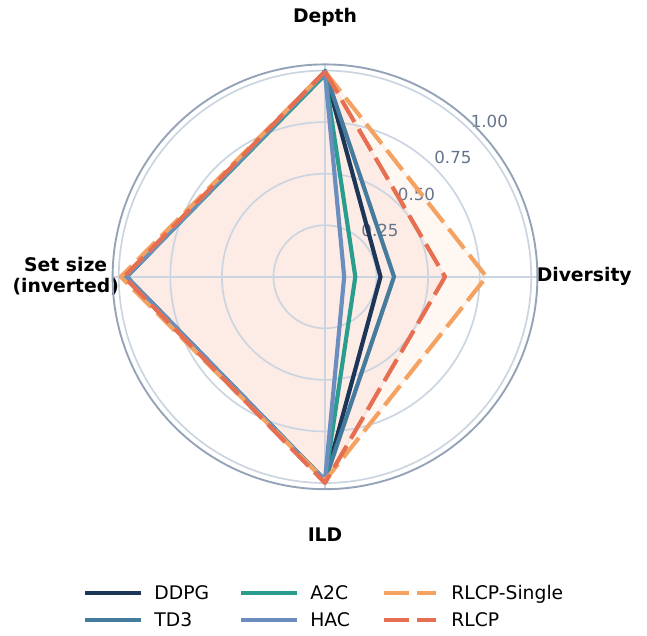}
   \caption{Performance across the 19 experimental configurations. The axes are normalized, and set size is inverted so that larger values are preferable on every axis. RLCP improves catalog diversity while maintaining competitive session depth under the same set size cap as the baselines. The full detailed results and configurations are in Appendix~\ref{app:full_results}.}
   \label{fig:overall_comparison}
   \vspace{-3ex}
\end{figure}

{\bf Main contributions.}
\begin{enumerate}[leftmargin=*]
\vspace{-2ex}
    \item We develop RLCP to construct action sets through learned scores and online retention calibration. A bilevel formulation evaluates set size and retention under the policy induced by the retained actions, accounting for the effect of filtering on subsequent state visitation.
    \item We establish a deterministic proxy miss bound and an exact value decomposition. Together, they give conditional bounds on discounted value and finite session reward that quantify the effects of proxy error, critic error, selection error, and set truncation.
    \item We evaluate two RLCP implementations against four RL baselines on KuaiRand-Pure and MovieLens 1M. At least one variant achieves the highest catalog diversity in every configuration, with competitive session depth and mean retained sets no larger than the corresponding baseline slates.
    \vspace{-1ex}
\end{enumerate}

The conformal component builds on existing online calibration methods; our focus is how calibrated retention interacts with action filtering and sequential value. Appendix~\ref{app:related_work} places this contribution in the context of conformal recommendation and decision-making.

\begin{figure*}[t]
    \centering
    \includegraphics[width=0.95\textwidth]{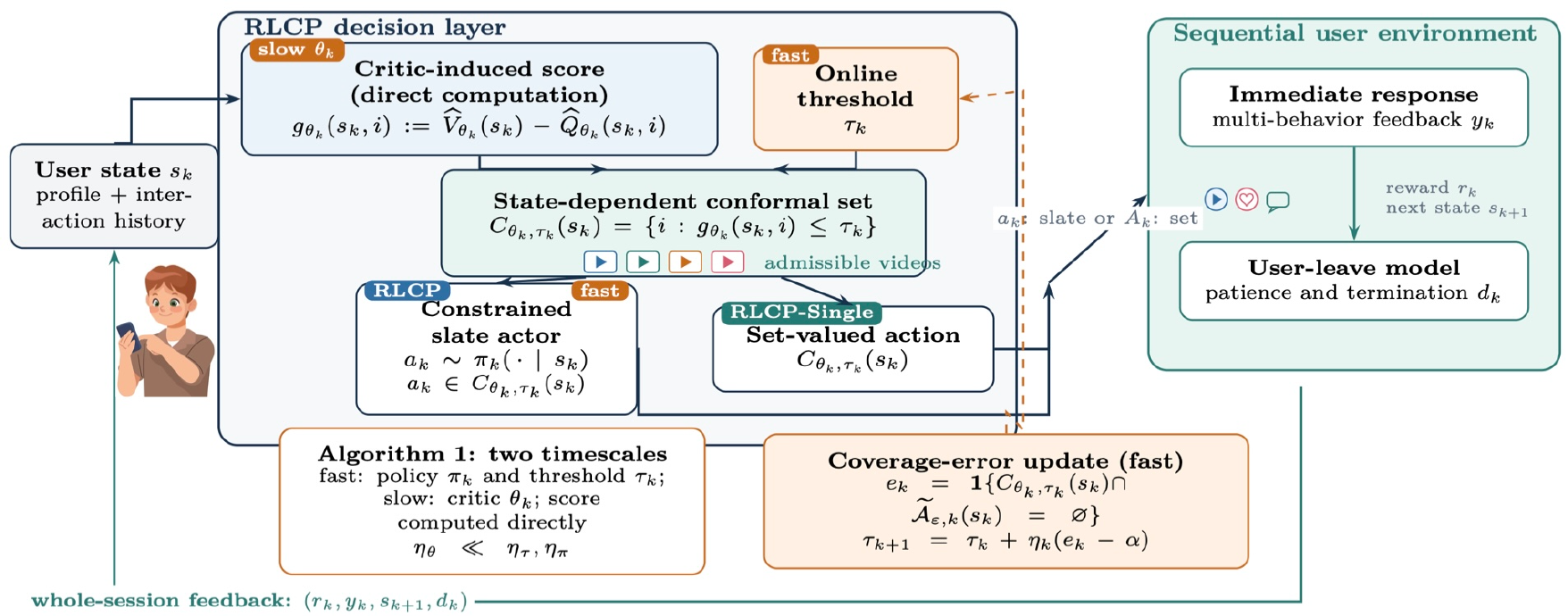}
    \caption{RLCP framework. At state $s_k$, critic scores $g_{\theta_k}(s_k,a)$ and threshold $\tau_k$ define the raw set $C_{\theta_k,\tau_k}(s_k)$. A minimum score fallback ensures nonempty execution. The resulting set constrains the downstream actor in RLCP or supplies the simulator in RLCP Single. Proxy miss feedback updates the threshold more frequently than the critic is updated; this schedule is an implementation choice.}
    \label{fig:rlcp_overview}
    \vspace{-3ex}
\end{figure*}

\section{{Problem Formulation}}\label{sec:formulation}
We consider a finite discounted MDP $\mathcal M=(\mathcal S,\mathcal A,\allowbreak P,r,\gamma,\mu_0)$, where $\mathcal S$ and $\mathcal A$ are finite nonempty state and action spaces, $P$ is the transition kernel, $r(s,a)\in[0,1]$ is the expected reward, $\gamma\in[0,1)$ is the discount factor, and $\mu_0$ is the initial state distribution. At time $t$, the learner observes $s_t$, selects $a_t$, receives $R_t\in[0,1]$, and observes $s_{t+1}$. Let $\mathcal H_t$ contain the trajectory and all algorithmic information available before $a_t$ is sampled. The environment satisfies
\begin{equation}
\begin{aligned}
\Pr(s_{t+1}=s'\mid\mathcal H_t,a_t=a)&=P(s'\mid s_t,a),\\
\mathbb E[R_t\mid\mathcal H_t,a_t=a]&=r(s_t,a),
\end{aligned}
\label{eq:controlled_markov_semantics}
\end{equation}
The learner may use the full history to update its policy, while the reward and transition laws depend only on the current state and action. Let $\Pi$ be the set of stationary Markov policies. For $\pi\in\Pi$, define the state value and state-action value functions as
\begin{equation}
\begin{aligned}
V^\pi(s)
=\mathbb E_{s}^{\pi}\!\left[\sum_{t=0}^{\infty}\gamma^tR_t\right],
Q^\pi(s,a)
=\mathbb E_{s,a\sim\pi}\!\left[\sum_{t=0}^{\infty}\gamma^tR_t\right].
\end{aligned}
\label{eq:value_definitions}
\end{equation}
 For a state distribution $\nu$, write $V^\pi(\nu)=\mathbb E_{s\sim\nu}[V^\pi(s)]$. The optimal values are $V^*(s)=\max_{\pi\in\Pi}V^\pi(s)$ and $Q^*(s,a)=\max_{\pi\in\Pi}Q^\pi(s,a)$, with $V^*(s)=\max_aQ^*(s,a)$ and $V^*(\nu)=\mathbb E_{s\sim\nu}[V^*(s)]$. We evaluate action sets under the normalized discounted state occupancy:
\begin{equation}
d^\pi(s)=(1-\gamma)\sum_{t=0}^{\infty}\gamma^t
\Pr^\pi(s_t=s\mid s_0\sim\mu_0).
\label{eq:state_occupancy}
\end{equation}
For a value tolerance $\varepsilon\ge0$, and $\alpha\in(0,1)$ is the target miss rate. Define the oracle target action set as
\begin{equation}
\mathcal A^*_\varepsilon(s)
=\{a\in\mathcal A:Q^*(s,a)\ge V^*(s)-\varepsilon\}.
\label{eq:oracle_action_set}
\end{equation}
An action belongs to this set if its optimal continuation value is within $\varepsilon$ of the best available value. This target specifies which actions should be retained, but it is unknown to the learner.

\vspace{-1ex}
\subsection{Action Filtering and Execution}
\vspace{-1ex}
Since the oracle target $\mathcal A^*_\varepsilon(s)$ depends on the unknown $Q^*$, the learner cannot use it directly to filter actions. Instead, it uses a learned score $g_\theta(s,a)\in\mathbb R$, with smaller scores preferred. For a threshold $\tau\in\mathbb R$, the raw set consists of all actions whose scores do not exceed the threshold:
\begin{equation}
C_{\theta,\tau}(s)=\{a:g_\theta(s,a)\le\tau\}.
\label{eq:raw_action_set}
\end{equation}
When no action passes the threshold, the policy still needs an action to execute. We therefore fix an order on $\mathcal A$ and use the action with the smallest score as a fallback:
\begin{equation}
a_\theta^{\min}(s)
=\text{first action in }\argmin_a g_\theta(s,a).
\label{eq:fallback_action}
\end{equation}
The resulting execution set is always nonempty:
\begin{equation}
\overline C_{\theta,\tau}(s)=
\begin{cases}
C_{\theta,\tau}(s),&C_{\theta,\tau}(s)\ne\emptyset,\\
\{a_\theta^{\min}(s)\},&C_{\theta,\tau}(s)=\emptyset.
\end{cases}
\label{eq:operational_action_set}
\end{equation}

The execution set determines which actions the downstream policy may select. To express this restriction, for any set map $D$ satisfying $\emptyset\ne D(s)\subseteq\mathcal A$ at every state, let
$\Pi(D)=\{\pi\in\Pi:\operatorname{supp}\pi(\cdot\mid s)\subseteq D(s),\ \forall s\in\mathcal S\}$
denote the stationary policies supported on $D$. For the construction above, we therefore consider a downstream policy $\pi\in\Pi(\overline C_{\theta,\tau})$.

In recommendation, $D(s)$ represents an intermediate candidate set and need not coincide with the displayed slate. The item MDP applies when an item is selected before the current feedback is generated and, conditional on the current state and selected item, reward and transition no longer depend on the remaining collection. Appendix~\ref{subsec:recommendation_interpretation} formalizes this requirement and explains when the collection itself must instead be modeled as the action.

With execution specified, we next evaluate whether the threshold rule retains a valuable action. This evaluation uses the raw set $C_{\theta,\tau}(s)$: retention succeeds when that set contains at least one action in the oracle target $\mathcal A^*_\varepsilon(s)$. The fallback serves a different purpose, namely, ensuring that execution remains possible. Thus an empty raw set counts as a retention failure even if the fallback action belongs to the oracle target, because no target action passed the threshold. Whether retention succeeds depends on the state. We therefore evaluate it under the normalized discounted occupancy $d^\pi$ of the downstream policy $\pi$. For a target miss rate $\alpha\in(0,1)$, the desired retention condition is
\begin{equation}
\Pr_{s\sim d^\pi}\!\left(
C_{\theta,\tau}(s)\cap\mathcal A^*_\varepsilon(s)\ne\emptyset
\right)\ge1-\alpha.
\label{eq:coverage_constraint}
\end{equation}
The condition allows the states where every oracle target action is removed to carry at most $\alpha$ of the discounted occupancy mass. The policy generating this occupancy is restricted by the execution set, so changing the threshold may change both the retained actions and the states on which retention is evaluated.

Let $m_\theta(s)=\min_{a\in\mathcal A^*_\varepsilon(s)}g_\theta(s,a)$ be the smallest score among oracle target actions. Retention succeeds exactly when this score passes the threshold:
\begin{equation}
C_{\theta,\tau}(s)\cap\mathcal A^*_\varepsilon(s)\ne\emptyset
\quad\Longleftrightarrow\quad m_\theta(s)\le\tau.
\label{eq:hit_margin_equivalence}
\end{equation}
This gives a scalar condition for threshold selection. The oracle margin is unavailable because its target depends on $Q^*$; RLCP will instead calibrate a proxy margin obtained from observable feedback.
\vspace{-1ex}
\subsection{Policy-Dependent Action-Set Objective}
\label{sec:objective}
\vspace{-1ex}
A threshold determines which actions remain available throughout the session. The set size objective must therefore account for the policy that uses those actions. Fix a learned score $g_\theta$ and write $D_{\theta,\tau}=\overline C_{\theta,\tau}$. The best policy supported on this set has value
\begin{equation}
V^*_{\theta,\tau}(\mu_0)
=\max_{\pi\in\Pi(D_{\theta,\tau})}V^\pi(\mu_0).
\label{eq:restricted_optimal_value}
\end{equation}
We denote the selected optimal response by $\pi^*_{\theta,\tau}$. A fixed lexicographic rule on optimal occupancies, together with a fixed action at states of zero occupancy, resolves ties; Appendix~\ref{app:additional_theory} gives the construction. We then minimize expected set size subject to retention under this response:

\begin{problem}[Occupancy-dependent threshold selection]
\label{prob:bilevel_corrected}
For the score $g_\theta$, consider
\begin{equation}
\begin{aligned}
\min_{\tau\in\mathbb R}\quad
&\Phi_\theta(\tau)
:=\mathbb E_{s\sim d^{\pi^*_{\theta,\tau}}}
[|D_{\theta,\tau}(s)|]\\
\text{s.t. }\quad
&\Pr_{s\sim d^{\pi^*_{\theta,\tau}}}\!\left(
C_{\theta,\tau}(s)\cap\mathcal A^*_\varepsilon(s)\ne\emptyset
\right)\ge1-\alpha,\\
&\pi^*_{\theta,\tau}\in
\argmax_{\pi\in\Pi(D_{\theta,\tau})}V^\pi(\mu_0).
\end{aligned}
\label{eq:bilevel_objective}
\end{equation}
\end{problem}

The inner problem determines how the retained actions are used, and the outer problem evaluates set size and retention under the resulting state occupancy. Consequently, a larger threshold may induce a different policy and change which states receive weight in the objective. The smallest feasible threshold need not minimize expected set size, unlike threshold selection under a fixed evaluation distribution.

The score is fixed only when defining this benchmark. Allowing arbitrary scores in the outer optimization would admit the uninformative oracle solution of assigning the smallest score to one optimal action at each state and returning singleton sets. RLCP instead learns scores by value estimation and updates the threshold from proxy feedback. We do not assume that these coupled updates solve~\eqref{eq:bilevel_objective}; the following guarantees apply while the score and policy continue to change.

\section{Main Results}
\label{sec:main_results}
In this section, we formally present our main results. We first bound the proxy miss frequency under the online threshold update. We then derive a value identity that separates the loss from filtering from the loss of selecting an action within the retained set. Combining the two results gives a reward bound for a policy that adapts during a finite session. Appendix~\ref{app:additional_theory} establishes the population benchmark, and the detailed proofs are deferred to Appendix~\ref{app:proofs}.

\vspace{-1ex}
\subsection{Retention Calibration During Learning}
\vspace{-1ex}
When both the score and policy are fixed, the smallest threshold satisfying the retention requirement is the $(1-\alpha)$ quantile of the oracle margin (Proposition~\ref{prop:quantile}). During learning, however, updates to the score and policy can change the margin distribution, so a threshold calibrated at one stage need not remain valid later. To account for these changes, we adjust the threshold using observed misses rather than assuming a fixed margin distribution. A miss increases the threshold to make filtering less restrictive, whereas a hit decreases it. Specifically, we use the online calibration update~\citep{gibbs2021adaptive,angelopoulos2024online}
\begin{equation}
e_t=\Ind\{m_t>\tau_t\},\qquad
\tau_{t+1}=\tau_t+\eta_t(e_t-\alpha).
\label{eq:online_threshold_update}
\end{equation}
Here $m_t$ is a finite margin, $e_t$ is its miss indicator, and $\eta_t>0$ is a nonincreasing step size. A miss raises the threshold by $\eta_t(1-\alpha)$, making filtering less restrictive; a hit lowers it by $\eta_t\alpha$. Theorem~\ref{thm:pathwise} bounds the resulting miss frequency for any realized margin sequence.

\begin{theorem}[Pathwise miss-rate tracking]
\label{thm:pathwise}
For any adaptive sequence of finite margins, define
\begin{equation}
W_T=\max\{\tau_1,m_1,\ldots,m_T\}
-\min\{\tau_1,m_1,\ldots,m_T\}.
\label{eq:realized_width}
\end{equation}
Then, for every $T\ge1$,
\begin{equation}
\left|\frac1T\sum_{t=1}^Te_t-\alpha\right|
\le\frac{W_T+\eta_1}{\eta_TT}.
\label{eq:pathwise_calibration_bound}
\end{equation}
If $\sup_TW_T\le W<\infty$ and $\eta_TT\to\infty$, then $T^{-1}\sum_{t=1}^Te_t\to\alpha$. If $\eta_t\equiv\eta$, the always-valid finite-time form is
\begin{equation}
\left|\frac1T\sum_{t=1}^Te_t-\alpha\right|
\le\frac{W_T+\eta}{\eta T},
\label{eq:constant_step_calibration}
\end{equation}
and it is at most $(W+\eta)/(\eta T)$ under the uniform-width condition above.
\end{theorem}
With bounded margin width and a constant step size, Theorem~\ref{thm:pathwise} gives an $O(1/T)$ deviation of the observed miss frequency from $\alpha$. The argument is deterministic: it requires neither independent margins nor a stationary margin distribution. It controls the average along the realized sequence, rather than retention at every state or convergence to a fixed population quantile. To apply the update without observing $Q^*$, RLCP uses a finite nonempty proxy target $\widetilde{\mathcal A}_{\varepsilon,t}(s_t)$. Its margin and miss indicator are
\begin{equation}
\widetilde m_t
=\min_{a\in\widetilde{\mathcal A}_{\varepsilon,t}(s_t)}g_{\theta_t}(s_t,a),
\qquad
\widetilde e_t=\Ind\{\widetilde m_t>\tau_t\}.
\label{eq:proxy_margin}
\end{equation}
Setting $m_t=\widetilde m_t$ in Theorem~\ref{thm:pathwise} controls the observed proxy miss rate while the critic and actor change. The relation between this proxy target and actions with high optimal value determines what the guarantee implies for reward.

A cap on set size can introduce misses even when the raw set retains a proxy target action. Let $C_t^K$ keep the $K$ actions with the lowest scores in the raw set, or all actions if fewer than $K$ are available. A capped miss is either a raw miss or a raw hit for which truncation removes every proxy target action. Calibration controls the former; the latter depends on the ranking and the cap. Appendix~\ref{app:maximum_set_size} bounds the combined miss rate and characterizes feasibility under a fixed distribution.
\vspace{-1ex}
\subsection{ {From Retention to Long-Horizon Value}}
\vspace{-1ex}
Retaining a valuable action is not enough if the policy does not select it. Consider a single state with two actions giving rewards one and zero. Retaining both actions gives perfect retention, but a policy that always chooses the second receives zero reward (Proposition~\ref{prop:coverage_insufficient}). To account for this distinction, for a nonempty set map $D$ and policy $\pi\in\Pi(D)$, define
\begin{equation}
\begin{aligned}
\Delta_D(s)
&=V^*(s)-\max_{a\in D(s)}Q^*(s,a),\\
\zeta_{\pi,D}(s)
&=\max_{a\in D(s)}Q^*(s,a)
-\sum_{a\in D(s)}\pi(a\mid s)Q^*(s,a).
\end{aligned}
\label{eq:set_selection_losses}
\end{equation}
The filtering loss $\Delta_D$ is the gap between the best available action and the best retained action. The selection loss $\zeta_{\pi,D}$ is the remaining gap between that retained action and the policy's expected choice. The next identity shows how both losses accumulate under the states visited by the policy.

\begin{theorem}[Filtering--selection value identity]
\label{thm:value_decomposition}
For every set map $D$ satisfying $D(s)\ne\emptyset$ for every state and every stationary $\pi\in\Pi(D)$,
\begin{equation}
V^*(\mu_0)-V^\pi(\mu_0)
=\frac1{1-\gamma}\mathbb E_{s\sim d^\pi}
[\Delta_D(s)+\zeta_{\pi,D}(s)].
\label{eq:exact_value_decomposition}
\end{equation}
\end{theorem}

Theorem~\ref{thm:value_decomposition} identifies where retention enters the control problem: it bounds the filtering term, while the downstream policy determines the selection term. Both losses use the same unrestricted optimal continuation values $Q^*$. A policy optimal under future action restrictions can therefore still have positive selection loss; this term is not just an error in optimizing the actor.

For proxy retention to control filtering loss, the proxy must identify actions with small optimal value gaps. Suppose the proxy target consists of near maximizers of a score $\widetilde Q$, and its pairwise gaps approximate those of $Q^*$ up to a scale $\lambda>0$ and error $\delta_{\mathrm{prox}}$ (Assumption~\ref{ass:proxy_gap}). Proposition~\ref{prop:proxy_gap_transfer} gives
\[
\widetilde{\mathcal A}_{\varepsilon}(s)
\subseteq\mathcal A^*_{\lambda\varepsilon+\delta_{\mathrm{prox}}}(s).
\]
Thus a proxy hit retains an action whose optimal value gap is at most $\lambda\varepsilon+\delta_{\mathrm{prox}}$. For selection, a critic error bounded by $\delta_{\mathrm{crit}}$ relative to $Q^*$ and a maximization error $\delta_{\mathrm{sel}}(s)$ within the retained set give $\zeta_{\pi,D}(s)\le2\delta_{\mathrm{crit}}+\delta_{\mathrm{sel}}(s)$. Substituting these bounds into the identity yields Corollary~\ref{cor:combined_approximation}. The approximation conditions are needed in addition to calibration: a critic trained to evaluate the current policy need not approximate the unrestricted $Q^*$.

\vspace{-1ex}
\subsection{Adaptive Finite Session Guarantee}
\vspace{-1ex}
The preceding identity evaluates a stationary policy. During learning, however, the policy may change after every interaction. We handle these updates by applying the same loss decomposition at each stage and using the calibration bound along the resulting trajectory. Consider an undiscounted item MDP with $H$ stages, rewards in $[0,1]$, and optimal values $V_h^*,Q_h^*$. Define
\[
\begin{aligned}
\Delta_H^{\max}&=\max_{h,s,a}[V_h^*(s)-Q_h^*(s,a)],\\
\varepsilon_H&=\min\{\lambda\varepsilon+\delta_{\mathrm{prox}},\Delta_H^{\max}\}.
\end{aligned}
\]
Here $\Delta_H^{\max}$ bounds the loss of any action, while $\varepsilon_H$ bounds the loss of a retained proxy target under the stagewise proxy fidelity condition. Let $\overline\delta_{\mathrm{sel},H}$ be the average expected selection error measured by the critic within the retained set, and let $\overline h_H^K$ be the average probability that a raw proxy hit is lost under a cap of $K$ actions. Appendix~\ref{app:finite_session_conditions} states the full information and approximation conditions.

\begin{theorem}[Finite-session proxy-to-value guarantee]
\label{thm:finite_session}
Under Assumption~\ref{ass:finite_session_conditions}, let $\eta_1\ge\cdots\ge\eta_H>0$ and suppose the realized proxy-margin width is at most $W$ almost surely. For $V_1^{\mathrm{RLCP}}(\mu_0)=\mathbb E[\sum_{h=1}^HR_h]$,
\begin{equation}
\begin{aligned}
&V_1^*(\mu_0)-V_1^{\mathrm{RLCP}}(\mu_0)\\
&\quad\le H\Bigg[
(1-\alpha)\varepsilon_H+\alpha\Delta_H^{\max}
+2\delta_{\mathrm{crit}}+\overline\delta_{\mathrm{sel},H}\\
&\qquad{}+(\Delta_H^{\max}-\varepsilon_H)
\left(\frac{W+\eta_1}{\eta_HH}+\overline h_H^K\right)
\Bigg].
\end{aligned}
\label{eq:finite_session_value_bound}
\end{equation}
Without a size cap, take $K=|\mathcal A|$ and $\overline h_H^K=0$.
\end{theorem}

The first two terms bound filtering loss on proxy hits and misses at the target rate $\alpha$. The critic and selection errors account for the action chosen from the execution set. The final term adds the deviation from the target miss rate and the hits removed by the cap. All of these quantities are evaluated along the adaptive trajectory, so the bound does not require convergence of the critic, actor, or threshold. Its size still depends on the stated approximation errors.

The reward determines what this bound measures. When reward indicates a nonterminal state and departure follows the item dynamics, it bounds truncated session depth. A bound for click or like reward does not imply the same depth guarantee, and a departure rule that depends on the entire collection requires a collection action model. Appendix~\ref{subsec:recommendation_interpretation} discusses this distinction for recommendation.

\section{RLCP Algorithm}
\label{sec:algorithm}

RLCP constructs the retained set from a critic $\widehat Q_{\theta_t}$ and an online threshold, as shown in Figure~\ref{fig:rlcp_overview}. At time $t$, the score of an action is its estimated value gap from the best candidate:
\begin{equation}
g_{\theta_t}(s,a)
=\max_b\widehat Q_{\theta_t}(s,b)-\widehat Q_{\theta_t}(s,a).
\label{eq:critic_gap_score}
\end{equation}
Adding the same constant to all critic values at a state leaves these scores unchanged. Thresholding retains actions whose estimated gaps do not exceed $\tau_t$. If the raw set exceeds the cap, we keep the actions with the smallest gaps; if it is empty, the fallback supplies a critic maximizer. The downstream selector then chooses from this execution set. Algorithm~\ref{alg:rlcp} gives the complete procedure.

\begin{algorithm}[t]
\caption{RLCP with Online Action-Set Calibration}
\label{alg:rlcp}
\small
\begin{tabular}{@{}r@{\quad}p{0.89\linewidth}@{}}
 & \textbf{Input:} target miss rate $\alpha$; interaction budget $T$; threshold $\tau_1$; nonincreasing $\eta_t>0$; cap $K\ge1$; critic $\widehat Q_{\theta_1}$; selector $\pi_{\phi_1}$; proxy-target verifier.\\
1 & \textbf{for} $t=1,\ldots,T$ \textbf{do}\\
2 & \quad Observe $s_t$ and compute $g_{\theta_t}(s_t,\cdot)$ using~\eqref{eq:critic_gap_score}.\\
3 & \quad Form $C_t=\{a:g_{\theta_t}(s_t,a)\le\tau_t\}$.\\
4 & \quad Keep the $\min\{K,|C_t|\}$ lowest-score actions to obtain $C_t^K$.\\
5 & \quad Set $D_t=C_t^K$ if nonempty; otherwise set $D_t=\{a_{\theta_t}^{\min}(s_t)\}$.\\
6 & \quad Define the proxy target for the current state and obtain the raw miss $\widetilde e_t=\Ind\{C_t\cap\widetilde{\mathcal A}_{\varepsilon,t}(s_t)=\emptyset\}$.\\
7 & \quad Sample $a_t\sim\pi_{\phi_t}(\cdot\mid s_t,D_t)$; observe $R_t,s_{t+1}$.\\
8 & \quad Update $\tau_{t+1}=\tau_t+\eta_t(\widetilde e_t-\alpha)$.\\
9 & \quad Update the critic and selector from interaction data.\\
10 & \textbf{end for}\\
\end{tabular}
\end{algorithm}

{\bf Calibration feedback.}
The threshold update uses a binary signal indicating whether the raw set intersects the proxy target. This signal evaluates retention before the cap and fallback; neither operation changes the miss used for calibration. A verifier supplies the signal without revealing $Q^*$, and our simulator implements the verifier through its response model. For Theorem~\ref{thm:finite_session}, the target must be defined before the current action is sampled, although the label may be revealed later. A target defined from the current realized responses can still be calibrated pathwise, but it does not satisfy this information condition. Feedback on the executed item alone generally cannot identify the required set event, so deployment requires a suitable verifier or counterfactual response model.

{\bf Learning and selection.}
Interaction data update the critic and the selector, while raw proxy misses update the threshold. In our implementation, calibration is updated more frequently than the critic. The pathwise guarantee uses the unprojected update in~\eqref{eq:online_threshold_update}; replacing raw misses with capped misses is a different procedure.

The selector also determines whether pruning changes the executed action. The critic gap score, cap, and fallback all preserve a global critic maximizer. An exactly greedy selector using the same critic and tie rule therefore chooses the same action with or without pruning (Proposition~\ref{prop:greedy_preservation}). RLCP instead regulates the alternatives available to a stochastic actor, slate constructor, or user choice mechanism. The experiments compare two such implementations.

\section{Experiments}
\label{sec:experiments}

We evaluate RLCP on the whole-session sequential recommendation task provided by the KuaiSim simulator~\citep{zhao2023kuaisim}. This protocol emulates online policy interaction through full-session rollouts, in which KuaiSim generates user feedback, computes the corresponding rewards, and determines when the user leaves the session. Our experiments are designed to answer the following questions:
\begin{itemize}
\vspace{-2ex}
    \item \textbf{Q1:} Can RLCP improve long-term recommendation performance under simulator-based sequential interaction?
    \vspace{-1ex}
    \item \textbf{Q2:} Can RLCP produce reliable admissible action sets that contain high-quality recommendations while reducing the effective action space?
      \vspace{-1ex}
\end{itemize}
\vspace{-1ex}
\subsection{Datasets and Simulator}
\label{subsec:datasets}
\vspace{-1ex}
We evaluate RLCP in the KuaiSim whole-session setting on
KuaiRand-Pure~\citep{gao2022kuairand}, the dataset used to construct
the original simulator, and on the MovieLens 1M migration demonstrated
by KuaiSim~\citep{harper2015movielens,zhao2023kuaisim}. KuaiRand-Pure comes from Kuaishou logs in which uniformly sampled videos were inserted into ordinary feeds. KuaiSim contains six positive behaviors---click, long view, like, comment, follow, and forward---together with dislike and session departure. MovieLens 1M contains about 1M timestamped ratings. Following the KuaiSim migration, ratings above 3 are mapped to positive feedback (\emph{like}), ratings at most 3 to negative feedback (\emph{hate}), interactions are ordered chronologically, and day-level groups define sessions. Table~\ref{tab:dataset_stats} reports statistics for both datasets.

\begin{table}[h]
\centering
\small
\setlength{\tabcolsep}{2.5pt}
\caption{Statistics of the datasets used.}
\label{tab:dataset_stats}
\begin{tabular}{lrrrrr}
\hline
Dataset                                                  & Users  & Items & Interactions & Sessions & Density \\ \hline
\begin{tabular}[c]{@{}l@{}}KuaiRand\\ -Pure\end{tabular} & 27,077 & 7,551 & 1,436,609    & 246,738  & 0.70\%  \\
ML-1M                                                    & 6,400  & 3,706 & 1,000,208    & 16,629   & 4.22\%  \\ \hline
\end{tabular}
\end{table}

We modify KuaiSim's leave module to make session length depend on the responses elicited by a recommendation. At each request, the recommender observes a user profile and recent interaction history, presents an item collection, and receives feedback from the immediate response model. We use \texttt{is\_click} and \texttt{is\_like} as reward signals. In our simulator inspection, click responses varied little across users, while like responses were more sensitive to user preferences.

At step $t$, the leave module evaluates the admissible set $C_{\theta,\tau}(s_t)$ for RLCP and the displayed slate for a baseline with fixed slate size. Patience decreases by $p$ if the collection receives at least one positive response and by $q>p$ otherwise. The session ends when remaining patience falls below $1$. We test $(p,q)=(1,2)$ and $(0.4,2)$, so a failure to elicit a positive response incurs the larger decrement in both settings.

\begin{figure*}[t]
  \centering

  \begin{subfigure}[t]{0.245\textwidth}
    \centering
    \includegraphics[width=\linewidth]{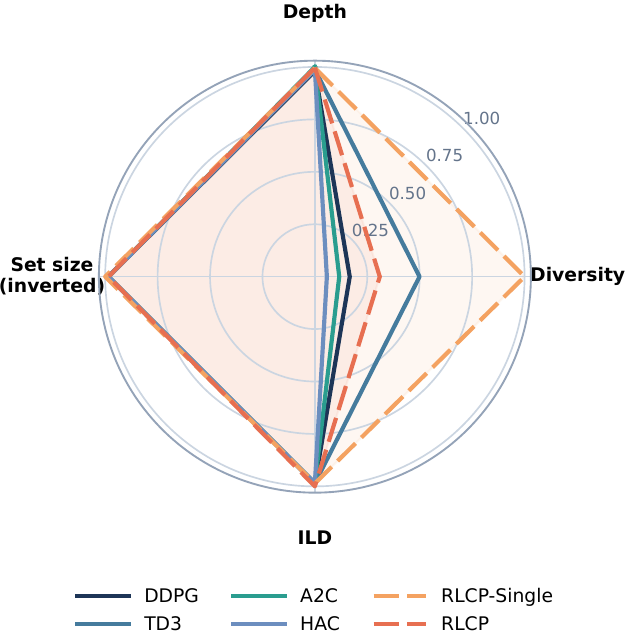}
    \caption{Is\_click, $p, q = (1,2)$.}
    \label{fig:kuairand_radar_1}
  \end{subfigure}\hfill
  \begin{subfigure}[t]{0.245\textwidth}
    \centering
    \includegraphics[width=\linewidth]{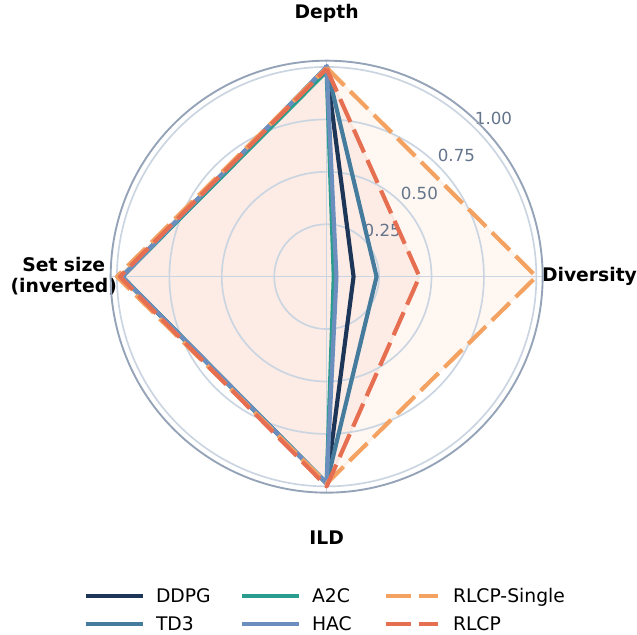}
    \caption{Is\_click, $p,q=(0.4,2)$.}
    \label{fig:kuairand_radar_2}
  \end{subfigure}\hfill
  \begin{subfigure}[t]{0.245\textwidth}
    \centering
    \includegraphics[width=\linewidth]{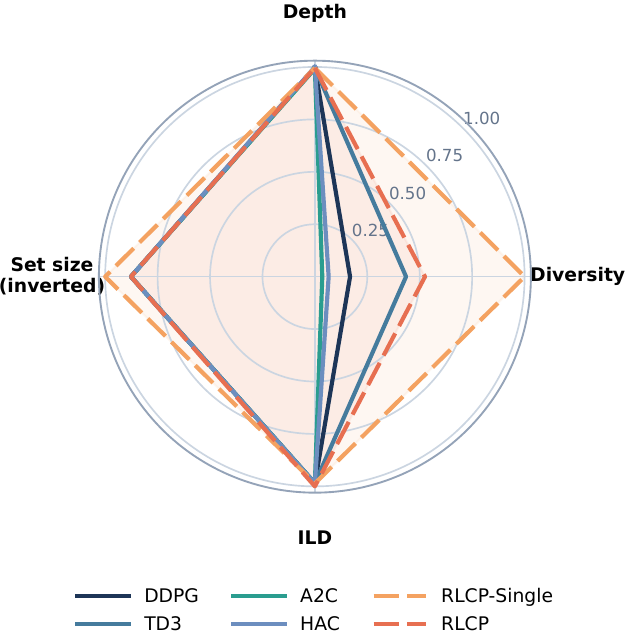}
    \caption{Is\_like, $p,q=(1,2)$.}
    \label{fig:kuairand_radar_3}
  \end{subfigure}\hfill
  \begin{subfigure}[t]{0.245\textwidth}
    \centering
    \includegraphics[width=\linewidth]{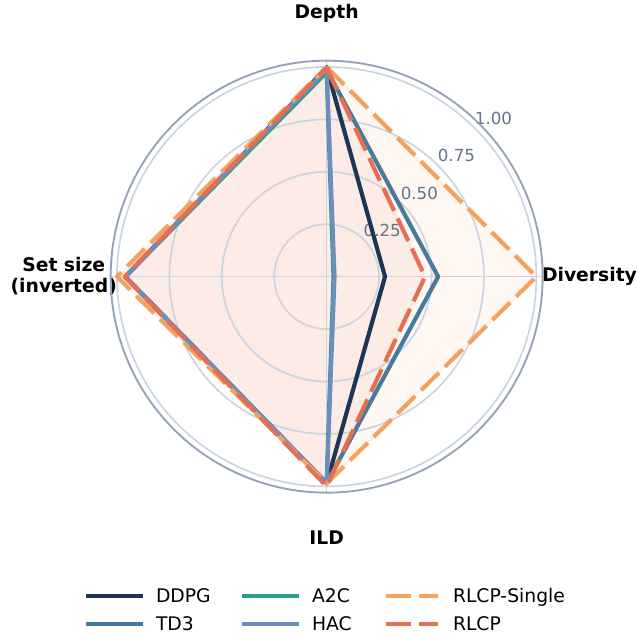}
    \caption{Is\_like, $p,q=(0.4,2)$.}
    \label{fig:kuairand_radar_4}
  \end{subfigure}

  \caption{KuaiRand-Pure radar summaries for the four reward and patience configurations reported in Tables~\ref{tab:kuairand_pq1_slate}--\ref{tab:kuairand_pq04_like_slate}. Each panel averages over the baseline slate sizes in the corresponding table and shows depth, catalog diversity, intra-list diversity, and inverted set size, with larger radius indicating better performance. RLCP-Single gives the strongest catalog-diversity gains on KuaiRand-Pure while using equal or smaller admissible sets, whereas the full RLCP consistently attains the highest intra-list diversity.}
  \label{fig:kuairand_setup_radars}
\end{figure*}

\vspace{-1ex}
\subsection{Baselines and Metrics}
\label{subsec:baselines}
\vspace{-1ex}
We compare RLCP with four whole-session sequential recommendation baselines provided by KuaiSim~\citep{zhao2023kuaisim}: A2C~\citep{mnih2016asynchronous}, DDPG~\citep{lillicrap2016continuous}, TD3~\citep{fujimoto2018addressing}, and HAC~\citep{liu2023exploration}; detailed descriptions are provided in Appendix~\ref{app:baseline}.

We report session depth, average and total reward, catalog diversity, intra-list diversity (ILD), and average admissible-set size, following KuaiSim~\citep{zhao2023kuaisim}. Depth is the primary session-level utility metric; catalog diversity measures the number of distinct items exposed across active user requests, ILD measures dissimilarity within a recommendation set, and set size measures filtering efficiency. Formal definitions are provided in Appendix~\ref{app:evaluation_metrics}.
\vspace{-1ex}
\subsection{Method Implementation}
\label{subsec:method_instantiation}
\vspace{-1ex}
We instantiate RLCP with two variants that share the conformal set construction, online calibration objective, and coverage target. The realized state $s_t$ concatenates the user profile with the full within-session history of rollout actions $(a_0,a_1,\ldots,a_{t-1})$, so the critic and rollout policy condition on all previously executed actions. The variants differ only in how the admissible set becomes a rollout action. Full RLCP follows Algorithm~\ref{alg:rlcp}: $C_{\theta,\tau}(s)$ constrains a downstream slate actor trained through sequential rollouts, while $\tau$ is updated online to maintain the target miss rate for near-optimal items. Critic/score updates occur on the slow timescale, and slate-actor rollouts and threshold calibration on the fast timescale.

\begin{figure*}[t]
  \centering

  \begin{subfigure}[t]{0.245\textwidth}
    \centering
    \includegraphics[width=\linewidth]{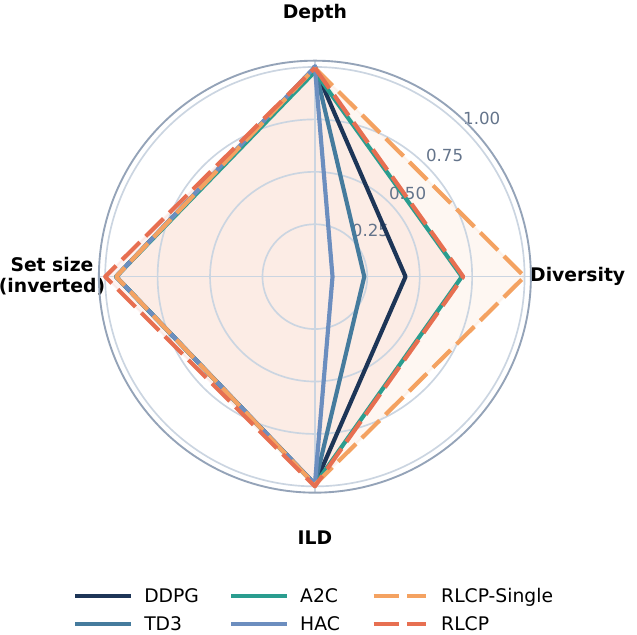}
    \caption{Is\_click, $p, q = (1,2)$.}
    \label{fig:ml1m_radar_1}
  \end{subfigure}\hfill
  \begin{subfigure}[t]{0.245\textwidth}
    \centering
    \includegraphics[width=\linewidth]{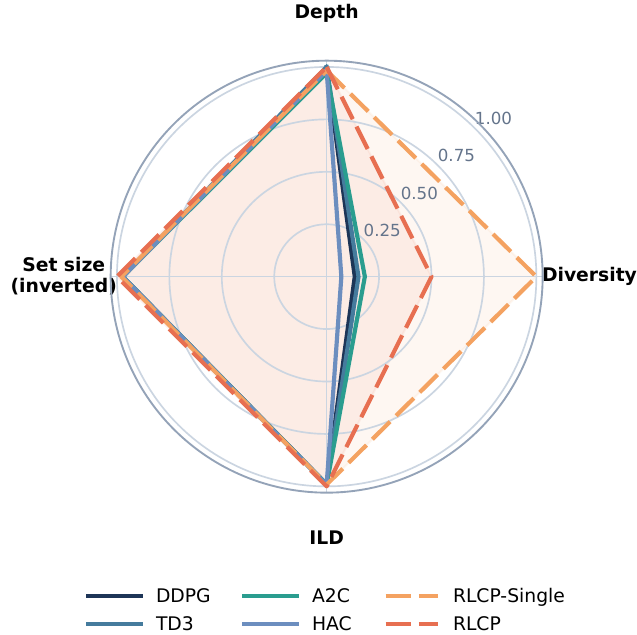}
    \caption{Is\_click, $p, q = (0.4,2)$.}
    \label{fig:ml1m_radar_2}
  \end{subfigure}\hfill
  \begin{subfigure}[t]{0.245\textwidth}
    \centering
    \includegraphics[width=\linewidth]{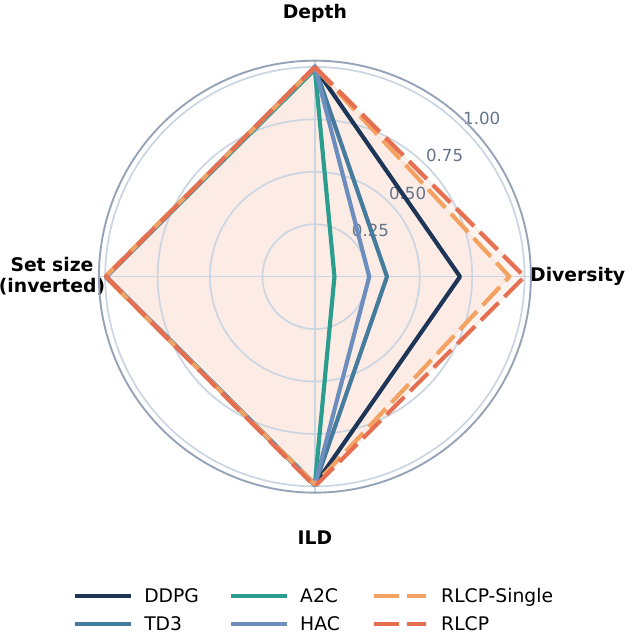}
    \caption{Is\_like, $p, q = (1,2)$.}
    \label{fig:ml1m_radar_3}
  \end{subfigure}\hfill
  \begin{subfigure}[t]{0.245\textwidth}
    \centering
    \includegraphics[width=\linewidth]{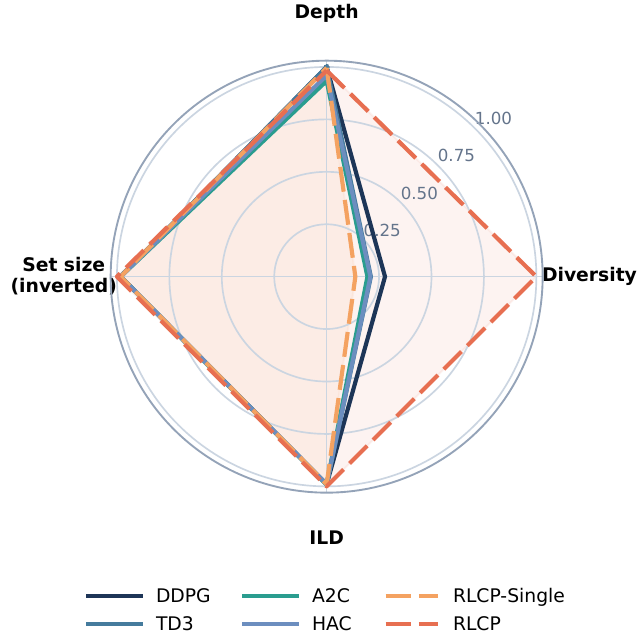}
    \caption{Is\_like, $p, q = (0.4,2)$.}
    \label{fig:ml1m_radar_4}
  \end{subfigure}

  \caption{ML-1M radar summaries for the four reward and patience configurations reported in Tables~\ref{tab:ml1m_pq1_slate}--\ref{tab:ml1m_pq04_like_slate}. Each panel averages over the baseline slate sizes in the corresponding table and uses the same normalized axes as Figure~\ref{fig:kuairand_setup_radars}. On the denser ML-1M data, the full RLCP is more competitive in terms of depth and catalog diversity while maintaining maximal intra-list diversity, and the RLCP variants retain their overall diversity advantage over the fixed-slate baselines.}
  \label{fig:ml1m_setup_radars}
  \vspace{-3ex}
\end{figure*}

RLCP-Single instead feeds the admissible set directly into the environment and uses the simulator's first rewarded item as the rollout action, imitating real-world user behavior. This simplifies rollout while preserving the online calibration objective. Both variants target $1-\alpha$ coverage of $\varepsilon$-optimal items, apply a minimum-set-size rescue when required (usually 1), and use the same conformal mask logic to bound set size.

{\bf Environment-Side Proxy.} To operationalize the $\varepsilon$-optimal action set in \eqref{eq:oracle_action_set}, we use the Bernoulli success probability $\hat p_b(s_t,a)$ from the simulator's immediate-response model and set
$
\widetilde Q(s_t,a)=\hat p_b(s_t,a),
$
which induces
\[
\widetilde{\mathcal A}_\varepsilon(s_t)
=
\left\{a:
\widetilde Q(s_t,a)
\ge
\max_j\widetilde Q(s_t,j)-\varepsilon
\right\}
\]
as in \eqref{eq:proxy_target_set}.
This conditional one-step expected reward is biased for the long-horizon $Q^*$ in \eqref{eq:oracle_action_set}. The numerical value of $\widetilde Q$ is \textit{never} exposed to the actor or critic; the simulator supplies only binary feedback indicating whether $C_{\theta_t,\tau_t}(s_t)\cap\widetilde{\mathcal A}_{\varepsilon}(s_t)$ is nonempty. This emulates positive user feedback when a recommendation set contains a satisfactory item.

{\bf Results.} Across all 19 configurations, {at least one RLCP variant achieves the highest catalog diversity without requiring a larger retained set}. Within each configuration, the RLCP cap equals the baseline slate size $M$, and all methods share the simulator, reward signal, and patience setting. Appendix Tables~\ref{tab:kuairand_pq1_slate} through~\ref{tab:ml1m_pq04_like_slate} provide the full numerical results. These comparisons show broader catalog exposure with competitive session depth under matched set size budgets.

As shown in Figure~\ref{fig:kuairand_setup_radars}, RLCP-Single provides the strongest catalog exposure on KuaiRand-Pure. It achieves the highest catalog diversity in all ten configurations, reaching {$1.48\times$ to $4.56\times$ that of the strongest baseline}. Its mean retained set is smaller than $M$ in nine configurations, with a maximum reduction of $25.5\%$, and equals $M$ in the remaining configuration. It also attains or ties the highest session depth in six configurations and stays within $0.4$ interactions of the best method in the others. Full RLCP achieves the highest reported ILD of $1.00$ throughout, but RLCP-Single generally provides a better balance between catalog exposure and session utility.

As shown in Figure~\ref{fig:ml1m_setup_radars}, the exposure advantage persists on ML-1M, where full RLCP becomes more competitive. It leads catalog diversity in five configurations, exceeds the strongest baseline in eight, and attains or ties the highest session depth in seven. It also maintains ILD $1.00$ throughout and reduces mean set size below $M$ in four configurations. RLCP-Single leads catalog diversity in the remaining four, giving the better variant in each configuration {$1.11\times$ to $5.21\times$ the catalog diversity of the strongest baseline}. Under the longer patience setting, this range is $3.17\times$ to $5.21\times$, although several methods reach the depth ceiling of $48$. Thus the main benefit is broader exposure with competitive depth, rather than uniformly higher reward or a single variant that dominates every setting.

\section{Conclusion}

We introduced RLCP, a conformal policy-learning framework for whole-session sequential recommendation. RLCP decouples critic-induced item scores from the calibrated threshold, producing state-dependent, variable-size sets rather than fixed-cardinality top-$M$ slates. The threshold provides pathwise control of the realized average miss rate, while two-timescale learning jointly updates the score model, calibration threshold, and rollout policy. Across KuaiRand-Pure and ML-1M, RLCP remains competitive with fixed-slate baselines in session depth while using sets of equal or smaller size and improving diversity. An RLCP variant achieves the highest catalog diversity in every configuration, and full RLCP consistently achieves the highest within-slate diversity. These results support conformal policy learning for retaining near-optimal recommendations, controlling set size, and broadening exposure without fixing slate cardinality in advance.

\bibliography{main}
\bibliographystyle{apalike}

\appendix
\thispagestyle{empty}

\onecolumn
\aistatstitle{
Supplementary Materials}

\section{Related Work}
\label{app:related_work}

RLCP draws on sequential control, conformal recommendation, and online calibration. The distinction across these areas is what is calibrated and how that quantity is connected to the subsequent decision. We organize the comparison around this distinction rather than treating the combination of RL and calibration as new.

{\bf RL for Sequential Recommendation.} RL recommender systems optimize utility across a sequence of interactions rather than only immediate feedback~\citep{afsar2022reinforcement,chen2021survey,zhao2018deep}. Within this framework, SlateQ makes slate optimization tractable by decomposing slate values under a user choice model~\citep{ie2019reinforcement}, while HAC learns continuous slate representations and a mapping from latent actions to item lists~\citep{liu2023exploration}. KuaiSim supplies whole session evaluation protocols and actor and critic baselines~\citep{zhao2023kuaisim}, and work on value estimation further distinguishes policy randomness from user response randomness~\citep{wang2025value}. These methods address how to evaluate or select recommendations. RLCP addresses a complementary question: which candidate actions should remain available to that selection rule, and how does their retention affect sequential value? Its item model does not cover every reward or transition rule that depends on a slate.

{\bf Conformal Recommendation.} Conformal methods provide a way to construct sets around a base prediction model while controlling a statistical error criterion~\citep{vovk2005algorithmic,angelopoulos2023gentle}. In recommendation, inductive conformal methods attach conformal $p$-values to candidate items~\citep{kagita2022inductive}, and distribution free recommendation sets calibrate ranking models against the false discovery rate (FDR)~\citep{angelopoulos2023recommendation}. Without explicitly optimizing long-term returns, CPFT incorporates conformal objectives into the fine tuning of sequential recommenders~\citep{wang2024confidence}, while conformal risk control has been used to limit unwanted recommendations~\citep{detoni2025flowers}. These approaches motivate adapting set size to a calibrated target. However, a recommendation target need not identify actions with high continuation value under an interacting policy. RLCP therefore distinguishes proxy retention from oracle retention and states the approximation condition needed to connect them.

{\bf Online Calibration and Decision Making.} Because RLCP changes the policy and the visited state distribution during learning, its calibration argument uses methods for changing data sequences. Adaptive conformal inference controls long-run empirical miscoverage under distribution changes~\citep{gibbs2021adaptive,gibbs2024online}. Multivalid methods provide richer guarantees across groups~\citep{bastani2022practical}, while methods with decaying step sizes combine guarantees for arbitrary sequences with population quantile consistency under stable sampling~\citep{angelopoulos2024online}. RLCP uses this established scalar update principle. Theorem~\ref{thm:pathwise} makes its dependence on the realized action margin range explicit and supplies the calibration term used in the value analysis. Online miscoverage tracking itself is not presented as a new principle.

Conformal methods also support motion planning~\citep{dixit2023adaptive}, return prediction for a policy different from the data collection policy~\citep{zhang2023conformal}, and cost aware orchestration of LLM agents~\citep{si2025ccpo}. In particular, CCPO jointly trains a policy and an adaptive conformal threshold, so RLCP is not distinguished by that broad combination. The distinction lies in the action filtering formulation and the analysis connecting retention to value. A proxy hit must first be related to an oracle hit, and the retained action must then be used by the selector. The value decomposition and finite session bound describe these two steps without assuming convergence of the coupled updates.

\section{Additional Modeling Details}
\label{app:modeling_details}

The main formulation treats an executed item as the MDP action. We now explain how this abstraction relates to a recommendation pipeline, where retrieval and collection based response rules introduce additional conditions.
\subsection{Sequential Recommendation Interpretation}
\label{subsec:recommendation_interpretation}
The first requirement concerns the state: it must contain the information needed to predict both immediate feedback and the next user state. A representative state is
\[
s_t=(x_u,h_t^a,h_t^y,\ell_t,c_t),
\]
Here $x_u$ is a user profile, $h_t^a$ and $h_t^y$ are action and response histories, $\ell_t$ is remaining patience, and $c_t$ is request context. An absorbing state represents user departure. The finite state model is an analytical idealization; in a large system, a learned representation must approximate the same predictive information.

The second requirement concerns retrieval. The finite action space can represent an eligible item pool, but pruning within that pool cannot recover an action already removed by retrieval. To separate these two losses, fix a state distribution $\nu$, let $\mathcal A_{\mathrm{ret}}(s)\subseteq\mathcal A$ be a finite nonempty retrieved pool, and let $C(s)\subseteq\mathcal A_{\mathrm{ret}}(s)$ be the raw set retained from it. Fix $\varepsilon_{\mathrm{ret}},\varepsilon_{\mathrm{filt}}\ge0$ and $\rho\in[0,1)$, and define
\begin{align*}
G_{\mathrm{ret}}(s)
&=\left\{
\max_{a\in\mathcal A_{\mathrm{ret}}(s)}Q^*(s,a)
\ge V^*(s)-\varepsilon_{\mathrm{ret}}
\right\},\\
G_{\mathrm{filt}}(s)
&=\bigl\{\exists a\in C(s)\text{ such that}\notag\\[-0.25ex]
&\hspace{2em}Q^*(s,a)\ge
\max_{b\in\mathcal A_{\mathrm{ret}}(s)}Q^*(s,b)
-\varepsilon_{\mathrm{filt}}
\bigr\}.
\end{align*}
If $\Pr_{s\sim\nu}(G_{\mathrm{ret}}^c)\le\rho$ and
$\Pr_{s\sim\nu}(G_{\mathrm{filt}}^c\mid G_{\mathrm{ret}})\le\alpha$, then on
$G_{\mathrm{ret}}\cap G_{\mathrm{filt}}$ the retained action has global optimality gap at most
$\varepsilon_{\mathrm{ret}}+\varepsilon_{\mathrm{filt}}$. The law of total probability gives
\[
\Pr_{s\sim\nu}\!\left(
C(s)\cap
\mathcal A^*_{\varepsilon_{\mathrm{ret}}+\varepsilon_{\mathrm{filt}}}(s)
=\emptyset
\right)
\le\rho+(1-\rho)\alpha.
\]
Thus the value tolerances add, while the miss probability accounts first for retrieval failure and then for filtering failure conditional on retrieval success. To retain the global tolerance $\varepsilon$, the two tolerances must sum to at most $\varepsilon$. Alternatively, filtering can be calibrated directly against
$\mathcal A^*_{\varepsilon}(s)\cap\mathcal A_{\mathrm{ret}}(s)$ conditional on that intersection being nonempty.

The final requirement concerns how a retained collection becomes an executed item. Let $\kappa(a\mid s,D)$ be a stationary Markov choice kernel supported on $D$, sampled before the current reward and next state are generated. This kernel may combine a downstream actor, slate construction, and user choice. For a stationary set map $D$, the induced item policy is
\begin{equation}
\pi_{D,\kappa}(a\mid s)=\kappa(a\mid s,D(s)).
\label{eq:induced_recommendation_policy}
\end{equation}

For this policy to reproduce the recommendation process, the selected item must account for the dependence of the current reward and next state on the collection. The following assumption states that requirement conditional on all information available before selection.

\begin{assumption}[Pre-response item-mediated dynamics]
\label{ass:item_mediation}
There is an item-level reward and transition kernel $K(\mathrm d r,s'\mid s,a)$ such that, under the pipeline law $\Pr\nolimits_D$ induced by every stationary set map $D$ with $D(s)\ne\emptyset$ for all $s$, the following factorization holds conditional on the full pre-action history $\mathcal H_t$:
\begin{equation}
\begin{split}
&\Pr\nolimits_D\bigl(a_t=a,R_t\in\mathrm d r,s_{t+1}=s'
\mid\mathcal H_t\bigr)\\
&\qquad=\kappa(a\mid s_t,D(s_t))K(\mathrm d r,s'\mid s_t,a).
\end{split}
\label{eq:item_mediation_factorization}
\end{equation}
The kernel agrees with the item-level MDP primitives:
\begin{equation}
\begin{aligned}
\int K(\mathrm d r,s'\mid s,a)
&=P(s'\mid s,a),\\
\sum_{s'}\int r\,K(\mathrm d r,s'\mid s,a)
&=r(s,a).
\end{aligned}
\label{eq:item_kernel_marginals}
\end{equation}
Thus, conditional on the state and executed item, the reward and next-state distribution do not depend further on the admissible set or displayed slate.
\end{assumption}

\begin{proposition}[Recommendation pipeline reduction]
\label{prop:recommendation_reduction}
Under Assumption~\ref{ass:item_mediation}, $\pi_{D,\kappa}\in\Pi(D)$, and the recommendation pipeline has the same reward and state-transition law as the item-level MDP controlled by $\pi_{D,\kappa}$. Hence all restricted-control and value results below apply to the induced policy.
\end{proposition}

\begin{proof}
For every state, $\kappa(\cdot\mid s,D(s))$ is a probability distribution supported on $D(s)$, so~\eqref{eq:induced_recommendation_policy} defines a stationary policy in $\Pi(D)$. Conditional on the full pre-action history, factorization~\eqref{eq:item_mediation_factorization} gives the same one-step joint law of the action, reward, and next state as this item-level policy. Starting from the same initial distribution and applying this equality inductively establishes equality of the complete trajectory laws.
\end{proof}

The order in Assumption~\ref{ass:item_mediation} matters. For example, a rule that observes the current response vector and then selects the first rewarded item does not satisfy this order: conditioning on the selected item changes its response law because selection used that response. The process can instead be modeled with the retained set or ordered slate as the action and an environment kernel that jointly generates responses, the selected item, reward, and the next state. However, the item value results do not automatically transfer to this different MDP. The pathwise calibration result still concerns the realized proxy misses, and Proposition~\ref{prop:patience_loss} separately covers a leave rule based on any positive response in the collection.

In summary, the item abstraction applies when an item selected before feedback determines the next interaction. When position, competition, or exposure to unchosen items also changes reward or transition, the collection must instead be the primitive action. In either model, a value bound concerns the specified reward: a click or like reward does not by itself establish a depth bound without an appropriate relation to departure.

\section{Additional Supporting Results}
\label{app:additional_theory}

The supporting results follow the same order as the main argument. We first establish the population benchmark and its characterization under a fixed policy. We then account for a cap on set size and state the conditions that transfer proxy retention and approximate selection to value.
\subsection{Restricted Control and Existence}
Before optimizing a threshold, we must ensure that every nonempty execution set supports a well defined control problem. For a set map $D$ satisfying $D(s)\ne\emptyset$ for every state, define
\begin{equation}
(\mathcal T_Dv)(s)=\max_{a\in D(s)}\left[r(s,a)+\gamma\sum_{s'}P(s'\mid s,a)v(s')\right].
\label{eq:restricted_bellman_operator}
\end{equation}

This is the Bellman operator with maximization restricted to retained actions. The next result establishes its optimal value and an occupancy representation, which is also needed to select a unique response for the outer objective.

\begin{theorem}[Restricted control is well defined]
\label{thm:inner_wellposed}
For every set map $D$ satisfying $D(s)\ne\emptyset$ for all $s$, $\mathcal T_D$ is a $\gamma$ contraction and has a unique fixed point $V_D^*$. A stationary deterministic policy in $\Pi(D)$ attains this value. The lower problem also admits a normalized discounted occupancy-measure LP whose optimum is $(1-\gamma)V_D^*(\mu_0)$, with attainment, zero duality gap, and a deterministic lexicographic rule that selects one optimal occupancy.
\end{theorem}

Once that response is fixed, the outer problem reduces to comparing the finitely many set maps induced by the score thresholds.

\begin{proposition}[Fixed-score outer existence]
\label{prop:outer_existence}
For every fixed $\theta$, Problem~\eqref{eq:bilevel_objective} is feasible and attains its minimum over the thresholds induced by that score.
\end{proposition}
A threshold retaining every action is feasible, so the outer comparison is nonempty. These results establish the benchmark for a given score. They do not assert that critic or actor training finds its global solution.

\subsection{Threshold Characterization for a Fixed Policy}
The threshold problem becomes simpler when the evaluation policy is held fixed as well. Fix $\theta$ and a stationary policy $\pi$, and define the distribution and quantile of the oracle margin by
\begin{equation}
\begin{aligned}
F_{\theta,\pi}(x)
&=\Pr_{s\sim d^\pi}(m_\theta(s)\le x),\\
F^{-1}_{\theta,\pi}(u)
&=\inf\{x:F_{\theta,\pi}(x)\ge u\}.
\end{aligned}
\label{eq:margin_cdf}
\end{equation}

\begin{proposition}[Smallest fixed-policy threshold]
\label{prop:quantile}
For fixed $(\theta,\pi)$, the smallest threshold satisfying~\eqref{eq:coverage_constraint} is
\begin{equation}
\tau_{\min}(\theta,\pi)=F^{-1}_{\theta,\pi}(1-\alpha).
\label{eq:minimal_threshold}
\end{equation}
Among thresholds evaluated under the same policy, it also minimizes $\mathbb E_{d^\pi}[|\overline C_{\theta,\tau}(s)|]$.
\end{proposition}
The quantile characterization follows from two facts: a hit is a margin threshold event, and the retained sets grow with the threshold under the same score. It therefore yields a smallest valid threshold and a minimum expected execution set size under a fixed occupancy. The population benchmark remains coupled because its occupancy changes with the downstream response.

\subsection{Maximum Set Size}
\label{app:maximum_set_size}
The calibration result controls the raw set, but execution may impose a maximum size. Suppose the cap keeps at most $K$ actions with the smallest scores, using a fixed action order to resolve ties. Write the capped raw set as $C_t^K$, and distinguish all capped misses from the hits removed specifically by truncation:
\begin{equation}
\begin{aligned}
e_t^K&=\Ind\{C_t^K\cap\widetilde{\mathcal A}_{\varepsilon,t}(s_t)=\emptyset\},\\
h_t^K&=\Ind\{C_t\cap\widetilde{\mathcal A}_{\varepsilon,t}(s_t)\ne\emptyset,
\ C_t^K\cap\widetilde{\mathcal A}_{\varepsilon,t}(s_t)=\emptyset\}.
\end{aligned}
\label{eq:cap_miss}
\end{equation}
Then $e_t^K=\widetilde e_t+h_t^K$, because a capped miss is either already a raw miss or is created by truncation. Hence, when the threshold update uses the raw miss $\widetilde e_t$,
\begin{equation}
\frac1T\sum_{t=1}^Te_t^K
\le \alpha+\frac{\widetilde W_T+\eta_1}{\eta_TT}
+\frac1T\sum_{t=1}^Th_t^K,
\label{eq:capped_pathwise_bound}
\end{equation}
Here $\widetilde W_T$ is~\eqref{eq:realized_width} for the proxy margins. The last term is not an error in tracking the raw miss rate. It records whether the ranking and available capacity preserve the proxy actions that passed the threshold.

This distinction also determines whether the retention target is feasible after capping. In the finite state and action model, fix the score, proxy target, and state distribution. A finite threshold can meet the target after capping if and only if
\begin{equation}
\Pr\!\left(
\begin{gathered}
\text{some proxy target action ranks}\\
\text{among the first }K\text{ actions}
\end{gathered}
\right)\ge1-\alpha.
\label{eq:cap_feasibility}
\end{equation}
Necessity follows because increasing the threshold cannot move a target action into the first $K$ ranks. For sufficiency, a threshold at least as large as every score makes the capped set equal to the global top $K$ at each state. Thus a failed capacity condition cannot be repaired by calibration alone. For the same reason, an update using capped misses does not generally satisfy the scalar margin proof: that proof requires a target of rank at most $K$, so that a sufficiently large finite threshold can produce a hit.

\subsection{Additional Value Results}
We next connect retained actions to value. The first example shows why the selection loss cannot be omitted, even when the retention condition is perfect.

\begin{proposition}[Coverage alone is insufficient]
\label{prop:coverage_insufficient}
There is a one-state, two-action discounted MDP in which the retained set always contains an optimal action, but a supported policy has value gap $(1-\gamma)^{-1}$.
\end{proposition}
To quantify the filtering component separately, define the largest oracle action gap and the target tolerance capped at that gap:
\begin{equation}
\begin{aligned}
\Delta_{\max}&=\max_{s,a}[V^*(s)-Q^*(s,a)],\\
\varepsilon_0&=\min\{\varepsilon,\Delta_{\max}\}.
\end{aligned}
\label{eq:maximum_gap}
\end{equation}

\begin{corollary}[Coverage bound]
\label{cor:coverage_to_value}
Let $C(s)\subseteq D(s)$ for every state, assume $D(s)\ne\emptyset$ for every state, and let $\pi\in\Pi(D)$ satisfy
\begin{equation}
\Pr_{s\sim d^\pi}\!\left(C(s)\cap\mathcal A^*_\varepsilon(s)\ne\emptyset\right)
\ge1-\alpha.
\label{eq:generic_coverage_assumption}
\end{equation}
With $\overline\zeta_{\pi,D}=\mathbb E_{d^\pi}[\zeta_{\pi,D}(s)]$,
\begin{equation}
V^*(\mu_0)-V^\pi(\mu_0)
\le\frac{(1-\alpha)\varepsilon_0
+\alpha\Delta_{\max}+\overline\zeta_{\pi,D}}{1-\gamma}.
\label{eq:coverage_to_value_bound}
\end{equation}
\end{corollary}
On a hit, the filtering loss is at most the target tolerance; on a miss, it is at most the largest action gap. The corollary averages these two cases and retains the separate selection loss from Theorem~\ref{thm:value_decomposition}. The factor $(1-\gamma)^{-1}$ then accounts for accumulation along the discounted session.

Because the population benchmark imposes retention under its own downstream occupancy, the same bound applies directly to every feasible threshold in that benchmark.

\begin{corollary}[Bilevel specialization]
\label{cor:bilevel_value}
For the fixed score $\theta$, every feasible threshold $\tau$ in Problem~\eqref{eq:bilevel_objective} satisfies Corollary~\ref{cor:coverage_to_value} with $C=C_{\theta,\tau}$, $D=D_{\theta,\tau}$, and $\pi=\pi^*_{\theta,\tau}$.
\end{corollary}

\subsection{Practical Proxy, Critic, and Recommendation Errors}
The preceding corollary assumes oracle retention. To apply it using observable feedback, we first replace the oracle target by the near maximizers of a finite proxy score $\widetilde Q:\mathcal S\times\mathcal A\to\mathbb R$:
\begin{equation}
\widetilde{\mathcal A}_\varepsilon(s)
=\{a:\widetilde Q(s,a)\ge\max_b\widetilde Q(s,b)-\varepsilon\}.
\label{eq:proxy_target_set}
\end{equation}

The proxy target is nonempty because it contains every maximizer of the finite score. Its membership depends only on differences between actions at the same state, so absolute agreement with $Q^*$ is stronger than needed. This matters when, for example, the proxy is a response probability rather than a cumulative value. The following condition instead compares pairwise gaps after allowing for scale.

\begin{assumption}[Proxy gap fidelity]
\label{ass:proxy_gap}
There exist $\lambda>0$ and $\delta_{\mathrm{prox}}\ge0$ such that, for every state and pair of actions,
\begin{equation}
\left|
Q^*(s,a)-Q^*(s,b)
-\lambda\bigl(\widetilde Q(s,a)-\widetilde Q(s,b)\bigr)
\right|
\le\delta_{\mathrm{prox}}.
\label{eq:proxy_gap_fidelity}
\end{equation}
\end{assumption}

Under this condition, a small gap according to the proxy also implies a controlled gap according to the oracle.

\begin{proposition}[Proxy-to-oracle target transfer]
\label{prop:proxy_gap_transfer}
Under Assumption~\ref{ass:proxy_gap},
\begin{equation}
\widetilde{\mathcal A}_\varepsilon(s)
\subseteq
\mathcal A^*_{\lambda\varepsilon+\delta_{\mathrm{prox}}}(s)
\qquad\forall s.
\label{eq:proxy_set_inclusion}
\end{equation}
Consequently, a hit of the proxy target is a hit of the enlarged oracle target under any state distribution.
\end{proposition}

The condition has a concrete interpretation for the response proxy used in recommendation. An immediate response can differ from optimal action value because the action also changes the continuation state. The next result separates that continuation effect from response model error.

\begin{proposition}[Immediate-response proxy]
\label{prop:immediate_proxy}
Suppose $\widetilde Q(s,a)=r(s,a)$ is the expected immediate response. Define
\begin{equation}
\delta_{\mathrm{dyn}}
=\gamma\max_{s,a,b}
\left|
\sum_{s'}[P(s'\mid s,a)-P(s'\mid s,b)]V^*(s')
\right|.
\label{eq:continuation_heterogeneity}
\end{equation}
Then Assumption~\ref{ass:proxy_gap} holds with $\lambda=1$ and $\delta_{\mathrm{prox}}=\delta_{\mathrm{dyn}}$. More generally, suppose $\widetilde Q(s,a)=\widehat p(s,a)$ and
$\|\widehat p-r\|_\infty\le\delta_{\mathrm{resp}}$. Then the assumption holds with
$\lambda=1$ and
$\delta_{\mathrm{prox}}=\delta_{\mathrm{dyn}}+2\delta_{\mathrm{resp}}$.
Using the convention
$\TV(P_1,P_2)=\frac12\sum_{s'}|P_1(s')-P_2(s')|$, we also have
\begin{equation}
\delta_{\mathrm{dyn}}
\le\frac{\gamma}{1-\gamma}
\max_{s,a,b}\TV\!\left(P(\cdot\mid s,a),P(\cdot\mid s,b)\right).
\label{eq:continuation_tv_bound}
\end{equation}
\end{proposition}
Thus an immediate response score approximates optimal item gaps when actions have similar continuation effects and the response model is accurate. The term $\delta_{\mathrm{dyn}}$ records the first requirement, while $2\delta_{\mathrm{resp}}$ records the second. A proxy equal to the simulator's exact Bernoulli parameter has zero response model error within that simulated MDP, but it can still have a large continuation error. Proxy retention is therefore not automatically retention of actions with high long term value.

Having related proxy hits to oracle hits, we now account for approximate selection. The following bound combines proxy fidelity with a critic error relative to the optimal value and an error in maximizing that critic within the retained set.

\begin{corollary}[Combined practical errors]
\label{cor:combined_approximation}
Let $C(s)\subseteq D(s)$ for every state, assume $D(s)\ne\emptyset$ for every state, and let $\pi\in\Pi(D)$. Suppose
\begin{equation}
\Pr_{s\sim d^\pi}\!\left(
C(s)\cap\widetilde{\mathcal A}_\varepsilon(s)\ne\emptyset
\right)\ge1-\alpha,
\label{eq:proxy_coverage_assumption}
\end{equation}
Assumption~\ref{ass:proxy_gap} holds, $\|\widehat Q-Q^*\|_\infty\le\delta_{\mathrm{crit}}$, and
\begin{equation}
\max_{a\in D(s)}\widehat Q(s,a)
-\sum_{a\in D(s)}\pi(a\mid s)\widehat Q(s,a)
\le\delta_{\mathrm{sel}}(s).
\label{eq:estimated_selection_error}
\end{equation}
With $\overline\delta_{\mathrm{sel}}=\mathbb E_{d^\pi}[\delta_{\mathrm{sel}}(s)]$ and
\begin{equation}
\varepsilon_{\mathrm{rec}}
=\min\{\lambda\varepsilon+\delta_{\mathrm{prox}},\Delta_{\max}\},
\label{eq:recommendation_tolerance}
\end{equation}
we have
\begin{equation}
\begin{aligned}
V^*(\mu_0)-V^\pi(\mu_0)
&\le\frac{1}{1-\gamma}\Bigl[
(1-\alpha)\varepsilon_{\mathrm{rec}}+\alpha\Delta_{\max}\\
&\hspace{5.5em}{}
+2\delta_{\mathrm{crit}}+\overline\delta_{\mathrm{sel}}\Bigr].
\end{aligned}
\label{eq:combined_approximation_bound}
\end{equation}
\end{corollary}

The bound identifies what each component must supply: calibration controls proxy misses, proxy fidelity bounds retained action quality, and critic accuracy and selection quality control the remaining decision loss. These are conditional requirements. In particular, a critic that evaluates the current policy or a restricted policy class need not have small error relative to the unrestricted $Q^*$.

For the simulator, it is also useful to study immediate patience changes without claiming a complete session value guarantee. The next result distinguishes a rule that tests the entire collection for a positive response from a rule that tests only an item selected before the response.

\begin{proposition}[Response retention under two patience mechanisms]
\label{prop:patience_loss}
Suppose that, conditional on state $s$, a jointly distributed potential-response vector $(Y_a)_{a\in\mathcal A}$ is defined and that $D(s)$ is nonempty and determined before this vector is sampled. Let
\[
p(s,a)=\Pr(Y_a=1\mid s),
\qquad
p^*(s)=\max_a p(s,a).
\]
For a collection-based response rule, define
\[
p_D^{\mathrm{set}}(s)
=\Pr\!\left(
\bigcup_{a\in D(s)}\{Y_a=1\}
\,\middle|\,s
\right).
\]
For a pre-response item selector $\kappa$, assume that its internal randomization is conditionally independent of the potential-response vector given $s$ and $D(s)$, and define
\[
p_{D,\kappa}^{\mathrm{item}}(s)
=\sum_{a\in D(s)}\kappa(a\mid s,D(s))p(s,a)
\]
and
\[
\delta_{\mathrm{sel}}^p(s)
=\max_{a\in D(s)}p(s,a)-p_{D,\kappa}^{\mathrm{item}}(s).
\]
Suppose that, under a state distribution $\nu$, the set $D(s)$ contains an action satisfying
$p(s,a)\ge p^*(s)-\varepsilon$ with probability at least $1-\alpha$. Let
\[
\varepsilon_p=\min\{\varepsilon,1\},
\qquad
\overline\delta_{\mathrm{sel}}^p
=\mathbb E_\nu[\delta_{\mathrm{sel}}^p(s)].
\]
Then the collection-based rule satisfies
\begin{equation}
\mathbb E_\nu[1-p_D^{\mathrm{set}}(s)]
\le
\mathbb E_\nu[1-p^*(s)]
+(1-\alpha)\varepsilon_p+\alpha.
\label{eq:no_response_bound}
\end{equation}
No independence among the item responses is required. The pre-response item selector satisfies
\begin{equation}
\mathbb E_\nu[1-p_{D,\kappa}^{\mathrm{item}}(s)]
\le
\mathbb E_\nu[1-p^*(s)]
+(1-\alpha)\varepsilon_p+\alpha
+\overline\delta_{\mathrm{sel}}^p.
\label{eq:selected_no_response_bound}
\end{equation}
If patience decreases by $c_+$ after a positive response and by $c_->c_+$ otherwise, the excess expected one-step patience decrease relative to the best single item is at most
\begin{equation}
(c_--c_+)\bigl[(1-\alpha)\varepsilon_p+\alpha\bigr]
\label{eq:patience_decrement_bound}
\end{equation}
for the collection-based rule, and at most
\begin{equation}
(c_--c_+)\bigl[(1-\alpha)\varepsilon_p+\alpha
+\overline\delta_{\mathrm{sel}}^p\bigr]
\label{eq:selected_patience_decrement_bound}
\end{equation}
for the pre-response item selector.
\end{proposition}
The difference between the bounds again comes from selection. A positive response anywhere in the collection is at least as likely as a positive response from its best item, whereas a selector may choose a worse item and incur the additional selection term. Both conclusions concern one interaction. Extending either to session depth requires the full patience evolution, state transition, and departure model.

\subsection{Full Conditions for the Finite Session Bound}
\label{app:finite_session_conditions}

The discounted results above concern stationary policies. Theorem~\ref{thm:finite_session} instead allows the policy and score to change within a session. To make that extension precise, the following assumption groups the information available before selection, the calibration and execution rules, and the approximation errors. It imposes no convergence requirement on the parameters that generate these quantities.

\begin{assumption}[Finite-session information and approximation errors]
\label{ass:finite_session_conditions}
Consider an undiscounted episodic MDP with finite nonempty state and action spaces $\mathcal S$ and $\mathcal A$, initial state $s_1\sim\mu_0$, horizon $H$, rewards $R_h\in[0,1]$, optimal functions $V_h^*,Q_h^*$, and terminal value $V_{H+1}^*=0$. Let $\mathcal H_h$ be the sigma field containing the trajectory and all algorithmic information available immediately before the action at stage $h$ is sampled. In particular, $s_h$ is $\mathcal H_h$ measurable. Let $r_h:\mathcal S\times\mathcal A\to[0,1]$ and $P_h(\cdot\mid s,a)$ be the controlled stage reward mean and transition kernel, so that
\[
\begin{aligned}
\mathbb E[R_h\mid\mathcal H_h,a_h]
&=r_h(s_h,a_h),\\
\Pr(s_{h+1}=s'&\mid\mathcal H_h,a_h)\\[-0.25ex]
&=P_h(s'\mid s_h,a_h).
\end{aligned}
\]
Let
\[
\Delta_H^{\max}
=\max_{1\le h\le H}\max_{s,a}
\bigl[V_h^*(s)-Q_h^*(s,a)\bigr].
\]
At every stage $h$, suppose that the threshold $\tau_h$ and the action-indexed quantities
\[
\begin{gathered}
g_h(a)=g_{\theta_h}(s_h,a),\qquad
\widetilde Q_h(s_h,a),\\
\widehat Q_h(s_h,a),\qquad a\in\mathcal A,
\end{gathered}
\]
are finite and $\mathcal H_h$ measurable. Define the stage-dependent proxy target, proxy margin, raw set, and raw proxy miss by
\begin{align*}
\widetilde{\mathcal A}_{\varepsilon,h}(s_h)
&=\left\{a\in\mathcal A:
\widetilde Q_h(s_h,a)
\ge \max_{b\in\mathcal A}\widetilde Q_h(s_h,b)-\varepsilon
\right\},\\
\widetilde m_h
&=\min_{a\in\widetilde{\mathcal A}_{\varepsilon,h}(s_h)}g_h(a),\\
C_h
&=\{a\in\mathcal A:g_h(a)\le\tau_h\},\\
\widetilde e_h
&=\Ind\{\widetilde m_h>\tau_h\}
=\Ind\{C_h\cap\widetilde{\mathcal A}_{\varepsilon,h}(s_h)=\emptyset\}.
\end{align*}
Let the threshold follow
\begin{equation}
\tau_{h+1}=\tau_h+\eta_h(\widetilde e_h-\alpha),
\qquad h=1,\ldots,H,
\label{eq:finite_session_threshold_update}
\end{equation}
where $\eta_1\ge\eta_2\ge\cdots\ge\eta_H>0$ is deterministic.

Fix $K\in\{1,\ldots,|\mathcal A|\}$. Using the fixed action order to break ties, let $C_h^K$ contain the $\min\{K,|C_h|\}$ elements of $C_h$ having the smallest scores $g_h$. Let $a_h^{\min}$ be the first action in $\argmin_{a\in\mathcal A}g_h(a)$ and define
\[
D_h=
\begin{cases}
C_h^K,&C_h^K\ne\emptyset,\\
\{a_h^{\min}\},&C_h^K=\emptyset.
\end{cases}
\]
The tie-breaking rule makes $C_h^K$ and $D_h$ $\mathcal H_h$ measurable. Conditional on $\mathcal H_h$, the action is sampled from an $\mathcal H_h$ measurable probability kernel supported on $D_h$, before the current reward and next state are generated. In particular,
\[
\Pr(a_h\in D_h\mid\mathcal H_h)=1.
\]
Define
\begin{align*}
e_h^K
&=\Ind\{C_h^K\cap
\widetilde{\mathcal A}_{\varepsilon,h}(s_h)=\emptyset\},\\
h_h^K
&=\Ind\{C_h\cap
\widetilde{\mathcal A}_{\varepsilon,h}(s_h)\ne\emptyset,
\ C_h^K\cap
\widetilde{\mathcal A}_{\varepsilon,h}(s_h)=\emptyset\},\\
\widetilde W_H
&=\max\{\tau_1,\widetilde m_1,\ldots,\widetilde m_H\}
-\min\{\tau_1,\widetilde m_1,\ldots,\widetilde m_H\}.
\end{align*}
Assume that $\widetilde W_H\le W$ almost surely for a deterministic finite constant $W$. Also assume that there exist $\lambda>0$, $\delta_{\mathrm{prox}}\ge0$, and $\delta_{\mathrm{crit}}\ge0$ such that, almost surely, for every $h$ and every $a,b\in\mathcal A$,
\begin{equation}
\begin{split}
\Bigl|Q_h^*(s_h,a)-Q_h^*(s_h,b)
&-\lambda\bigl(\widetilde Q_h(s_h,a)\\[-0.25ex]
&\hspace{2.5em}{}-\widetilde Q_h(s_h,b)\bigr)\Bigr|
\le\delta_{\mathrm{prox}}.
\end{split}
\label{eq:finite_session_proxy_fidelity}
\end{equation}
\begin{equation}
\max_{a\in\mathcal A}
\left|\widehat Q_h(s_h,a)-Q_h^*(s_h,a)\right|
\le\delta_{\mathrm{crit}}.
\label{eq:finite_session_critic_fidelity}
\end{equation}
Finally, suppose that $\delta_{\mathrm{sel},h}$ is a nonnegative, integrable, $\mathcal H_h$ measurable random variable and that, almost surely,
\begin{equation}
\max_{a\in D_h}\widehat Q_h(s_h,a)
-\mathbb E\!\left[
\widehat Q_h(s_h,a_h)\mid\mathcal H_h
\right]
\le\delta_{\mathrm{sel},h}.
\label{eq:finite_session_selector_condition}
\end{equation}
Define
\[
\begin{aligned}
\varepsilon_H
&=\min\{\lambda\varepsilon+\delta_{\mathrm{prox}},
\Delta_H^{\max}\},\\
\overline\delta_{\mathrm{sel},H}
&=\frac1H\sum_{h=1}^H\mathbb E[\delta_{\mathrm{sel},h}],\\
\overline h_H^K
&=\frac1H\sum_{h=1}^H\mathbb E[h_h^K].
\end{aligned}
\]
Taking $K=|\mathcal A|$ represents the case without a cap and gives $\overline h_H^K=0$.
\end{assumption}

The information condition permits the critic and selector to depend on past observations. It excludes using the current action's realized response to define the target or execution set for that same action. The approximation conditions then apply at the visited states, with critic error measured relative to unrestricted $Q_h^*$ rather than the current policy value. Thus the theorem explains how these errors affect performance; it does not establish that an arbitrary learning update makes them small.

\subsection{Greedy Preservation and Conditional Policy Recovery}
\label{app:policy_recovery}

The previous results allow arbitrary supported selectors. For a selector greedy with respect to the scoring critic, however, a simpler boundary case applies: the critic gap score always gives a minimum score to a critic maximizer. To describe its consequences, fix finite critics $\widehat Q_t$, thresholds $\tau_t$, and caps $K_t\ge1$. Construct $D_t(s)$ by thresholding the critic gap, keeping at most $K_t$ actions with the smallest scores, and applying the minimum score fallback when needed. Use the same fixed action order for all ties.

The remaining performance error can then be expressed using the critic's pairwise error and the selector's gap within the retained set. For a stationary policy $\pi_t$ supported on $D_t$, define
\begin{align*}
\xi_t(s,a)&=\widehat Q_t(s,a)-Q^*(s,a),\\
\beta_t&=\max_{s,a,b}|\xi_t(s,a)-\xi_t(s,b)|,\\
\sigma_t(s)&=\max_{a\in D_t(s)}\widehat Q_t(s,a)
-\sum_{a\in D_t(s)}\pi_t(a\mid s)\widehat Q_t(s,a),\\
\overline\sigma_t&=\mathbb E_{s\sim d^{\pi_t}}[\sigma_t(s)].
\end{align*}
If any strictly suboptimal state--action pair exists, let
\[
\Delta_{\min}=\min_{s,a:\,Q^*(s,a)<V^*(s)}
[V^*(s)-Q^*(s,a)]>0.
\]
If all actions are optimal, set $\Delta_{\min}=+\infty$.

\begin{proposition}[Greedy preservation and conditional recovery]
\label{prop:greedy_preservation}
Every operational set $D_t(s)$ contains the first global maximizer of $\widehat Q_t(s,\cdot)$. Hence a greedy selector with the same tie-breaking rule returns the same action with or without filtering. Moreover,
\[
V^*(\mu_0)-V^{\pi_t}(\mu_0)
\le\frac{\overline\sigma_t+\beta_t}{1-\gamma},
\qquad
\beta_t\le2\|\widehat Q_t-Q^*\|_\infty.
\]
If $\beta_t<\Delta_{\min}$, every global critic maximizer is truly optimal. A policy greedy within $D_t$ is then optimal. Consequently, $\beta_t+\overline\sigma_t\to0$ implies value convergence; for exact greedy selection, $\beta_t\to0$ implies eventual optimality in a finite MDP. These are conditional implications, not convergence guarantees for the learning updates.
\end{proposition}

This result does not replace the calibration analysis. It identifies a case in which filtering leaves the chosen action unchanged, so any benefit from controlling the available alternatives must arise through a different downstream selection mechanism. The recovery implications remain conditional on critic and selector errors.

\section{Complete Proofs}
\label{app:proofs}

The proofs follow the dependencies in the main analysis. Bellman identities first establish the reference values. Restricted control and occupancy arguments then justify the population benchmark. The calibration proof controls the miss sequence, while the value proofs convert misses and selection errors into reward loss. Finally, the greedy preservation argument identifies the boundary case in which filtering does not change the selected action.

\subsection{Bellman Identities and Value Bounds}
The following identities connect the trajectory definitions of value to the Bellman operators used in the later proofs. The bounded reward assumption then supplies the largest possible action and policy gaps.

\begin{lemma}[Bellman identities]
\label{lem:bellman}
For every stationary policy $\pi$,
\begin{equation}
\begin{aligned}
Q^\pi(s,a)
&=r(s,a)+\gamma\sum_{s'}P(s'\mid s,a)V^\pi(s'),\\
V^\pi(s)
&=\sum_a\pi(a\mid s)Q^\pi(s,a).
\end{aligned}
\label{eq:policy_bellman_appendix}
\end{equation}
The optimal values satisfy
\begin{equation}
\begin{aligned}
Q^*(s,a)
&=r(s,a)+\gamma\sum_{s'}P(s'\mid s,a)V^*(s'),\\
V^*(s)&=\max_aQ^*(s,a).
\end{aligned}
\label{eq:optimal_bellman_appendix}
\end{equation}
Any stationary policy supported on
$\argmax_aQ^*(s,a)$ at every state is optimal.
\end{lemma}

\begin{proof}
We prove the policy identities first. Fix a stationary policy $\pi$, a state $s$, and an action $a$. By the forced-action definition,
\[
Q^\pi(s,a)
=
\mathbb E_s^{a\triangleright\pi}\!\left[
R_0+\sum_{t=1}^{\infty}\gamma^tR_t
\right].
\]
The controlled reward model gives
\[
\mathbb E_s^{a\triangleright\pi}[R_0]=r(s,a).
\]
For the remaining terms, condition on the next state and shift the time index:
\begin{align*}
\mathbb E_s^{a\triangleright\pi}\!\left[
\sum_{t=1}^{\infty}\gamma^tR_t
\right]
&=
\gamma\sum_{s'}P(s'\mid s,a)
\mathbb E_{s'}^\pi\!\left[
\sum_{j=0}^{\infty}\gamma^jR_j
\right]\\
&=
\gamma\sum_{s'}P(s'\mid s,a)V^\pi(s').
\end{align*}
The first equality uses~\eqref{eq:controlled_markov_semantics} and the fact that the policy follows $\pi$ from time one onward. Hence
\[
Q^\pi(s,a)
=r(s,a)+\gamma\sum_{s'}P(s'\mid s,a)V^\pi(s').
\]

Under $\pi$, the time-zero action has distribution $\pi(\cdot\mid s)$. The trajectory law is therefore the corresponding mixture of the forced-action laws:
\begin{align*}
V^\pi(s)
&=\sum_a\pi(a\mid s)
\mathbb E_s^{a\triangleright\pi}\!\left[
\sum_{t=0}^{\infty}\gamma^tR_t
\right]\\
&=\sum_a\pi(a\mid s)Q^\pi(s,a).
\end{align*}
This establishes~\eqref{eq:policy_bellman_appendix}.

We next establish the optimal identities. Define
\[
(\mathcal Tv)(s)
=
\max_{a\in\mathcal A}
\left[
r(s,a)+\gamma\sum_{s'}P(s'\mid s,a)v(s')
\right].
\]
For a stationary policy $\pi$, also define its Bellman operator by
\[
(\mathcal T^\pi v)(s)
=\sum_a\pi(a\mid s)
\left[
r(s,a)+\gamma\sum_{s'}P(s'\mid s,a)v(s')
\right].
\]
For arbitrary bounded functions $v,w$,
\begin{align*}
|(\mathcal Tv)(s)-(\mathcal Tw)(s)|
&\le
\max_a
\left|
\gamma\sum_{s'}P(s'\mid s,a)
[v(s')-w(s')]
\right|\\
&\le
\gamma\|v-w\|_\infty.
\end{align*}
Taking the maximum over $s$ gives
\[
\|\mathcal Tv-\mathcal Tw\|_\infty
\le\gamma\|v-w\|_\infty.
\]
Thus $\mathcal T$ has a unique fixed point.

Let $\overline V$ denote the unique fixed point of $\mathcal T$. Set
$v_0\equiv0$ and $v_H=\mathcal T^H0$. Contraction of $\mathcal T$ gives
\[
v_H\longrightarrow\overline V
\qquad\text{in the sup norm}.
\]
The iterates also have the standard finite-horizon interpretation: by induction,
$v_H(s)$ is the largest expected discounted reward collected during the first
$H$ decision stages.

We now identify $\overline V$ with the value $V^*$ defined as the maximum over stationary policies, without assuming in advance that a stationary optimal policy exists. Fix any stationary policy $\pi$. For every function $v$,
\[
\mathcal T^\pi v\le\mathcal Tv
\]
componentwise because an average over actions cannot exceed their maximum. Both operators are monotone, so induction from the zero function gives
\[
(\mathcal T^\pi)^H0
\le
\mathcal T^H0
=v_H
\qquad\text{for every }H.
\]
The policy operator $\mathcal T^\pi$ is a contraction, so its iterates converge to $V^\pi$. Letting $H\to\infty$ yields
\[
V^\pi\le\overline V.
\]
Since this holds for every stationary policy,
\[
V^*\le\overline V.
\]

For the reverse inequality, choose a stationary deterministic policy
$\overline\pi$ satisfying
\[
\overline\pi(s)\in
\argmax_{a\in\mathcal A}
\left[
r(s,a)+\gamma\sum_{s'}P(s'\mid s,a)\overline V(s')
\right]
\]
at every state. Such a choice exists because $\mathcal A$ is finite. Since
$\overline V=\mathcal T\overline V$, the greedy construction gives
\[
\mathcal T^{\overline\pi}\overline V
=
\mathcal T\overline V
=
\overline V.
\]
The operator $\mathcal T^{\overline\pi}$ has the unique fixed point
$V^{\overline\pi}$, and therefore
\[
V^{\overline\pi}=\overline V.
\]
Because $\overline\pi$ is stationary,
\[
V^*\ge V^{\overline\pi}=\overline V.
\]
Combining both inequalities proves
\[
V^*=\overline V=\mathcal TV^*.
\]
This argument simultaneously proves existence of a stationary deterministic optimal policy.

Define
\[
\overline Q(s,a)
=
r(s,a)+\gamma\sum_{s'}P(s'\mid s,a)V^*(s').
\]
It remains to connect this Bellman action value to the trajectory definition
$Q^*(s,a)=\max_\pi Q^\pi(s,a)$. We establish the two inequalities separately.
For every stationary policy $\pi$, the policy value satisfies
$V^\pi\le V^*$ componentwise. Therefore
\begin{align*}
Q^\pi(s,a)
&=
r(s,a)+\gamma\sum_{s'}P(s'\mid s,a)V^\pi(s')\\
&\le
r(s,a)+\gamma\sum_{s'}P(s'\mid s,a)V^*(s')\\
&=\overline Q(s,a).
\end{align*}
Taking the maximum over $\pi$ gives
\[
Q^*(s,a)\le\overline Q(s,a).
\]

For the reverse inequality, choose a stationary policy $\pi^*$ whose support at every state is contained in
\[
\argmax_b
\left[
r(s,b)+\gamma\sum_{s'}P(s'\mid s,b)V^*(s')
\right].
\]
Such a policy exists because the action set is finite. Its policy Bellman operator satisfies
\[
\mathcal T^{\pi^*}V^*=\mathcal TV^*=V^*.
\]
Since $\mathcal T^{\pi^*}$ is a contraction with unique fixed point
$V^{\pi^*}$, we have
\[
V^{\pi^*}=V^*.
\]
Consequently,
\[
Q^{\pi^*}(s,a)
=
r(s,a)+\gamma\sum_{s'}P(s'\mid s,a)V^*(s')
=
\overline Q(s,a).
\]
Because $Q^*$ is the maximum over policies,
\[
Q^*(s,a)\ge Q^{\pi^*}(s,a)=\overline Q(s,a).
\]
Combining the two inequalities gives
\[
Q^*(s,a)
=
r(s,a)+\gamma\sum_{s'}P(s'\mid s,a)V^*(s').
\]
Substituting this identity into the fixed-point equation for $V^*$ yields
\[
V^*(s)=\max_aQ^*(s,a).
\]

Finally, any stationary policy whose support is contained in
$\argmax_aQ^*(s,a)$ satisfies
\[
\mathcal T^\pi V^*=V^*.
\]
The unique fixed point of $\mathcal T^\pi$ is $V^\pi$, so
$V^\pi=V^*$. Such a policy is therefore optimal.
\end{proof}

\begin{lemma}[Value bounds]
\label{lem:value_bounds}
For every $\pi\in\Pi$, state $s$, and action $a$,
\[
0\le V^\pi(s),Q^\pi(s,a),V^*(s),Q^*(s,a)
\le\frac1{1-\gamma}.
\]
Consequently,
$0\le\Delta_{\max}\le(1-\gamma)^{-1}$.
\end{lemma}

\begin{proof}
For every sample path,
\[
0\le\sum_{t=0}^{\infty}\gamma^tR_t
\le\sum_{t=0}^{\infty}\gamma^t
=\frac1{1-\gamma}.
\]
Taking expectations under the corresponding controlled trajectory laws preserves both inequalities, proving the bounds for $V^\pi$ and $Q^\pi$. Taking maxima over policies gives the same bounds for $V^*$ and $Q^*$. Finally,
$V^*(s)=\max_bQ^*(s,b)\ge Q^*(s,a)$, so each difference in~\eqref{eq:maximum_gap} is nonnegative and at most $(1-\gamma)^{-1}$.
\end{proof}

\subsection{Proofs for the Bilevel Formulation}
We first establish a supported optimal policy and its occupancy representation. The occupancy formulation then provides a deterministic choice among optimal responses, so that the outer objective has a definite value for every threshold.

\begin{proof}[Proof of Theorem~\ref{thm:inner_wellposed}]
Fix a set map $D$ satisfying $D(s)\ne\emptyset$ for every state. We will prove that control over this set is equivalent to maximizing reward over supported discounted occupancies. For reference, define
\begin{equation}
\max_{\rho\in\mathcal D(D)}\sum_{s,a}\rho(s,a)r(s,a),
\label{eq:restricted_occupancy_lp}
\end{equation}
where
\begin{equation}
\begin{aligned}
\mathcal D(D)=\{\rho\ge0:\;&\sum_a\rho(s,a)
=(1-\gamma)\mu_0(s)\\
&+\gamma\sum_{s',a'}P(s\mid s',a')\rho(s',a'),\ \forall s,\\
&\rho(s,a)=0\text{ for }a\notin D(s)\}.
\end{aligned}
\label{eq:restricted_occupancy_set}
\end{equation}
The flow constraints record discounted state visitation and rule out actions outside $D$. The corresponding dual places an upper bound on reward plus continuation value for every retained action:
\begin{equation}
\begin{aligned}
\min_v\quad&(1-\gamma)\sum_s\mu_0(s)v(s)\\
\text{subject to}\quad&v(s)\ge r(s,a)+\gamma\sum_{s'}P(s'\mid s,a)v(s'),\\
&\hspace{3.45cm}\forall s,\ a\in D(s).
\end{aligned}
\label{eq:restricted_occupancy_dual}
\end{equation}

For each state $s$ and action $a\in D(s)$, define
\[
h_a(v)
=r(s,a)+\gamma\sum_{s'}P(s'\mid s,a)v(s').
\]
For any two finite families of real numbers $\{x_a\}$ and $\{y_a\}$,
\[
\left|\max_a x_a-\max_a y_a\right|
\le\max_a|x_a-y_a|.
\]
Applying this inequality with $x_a=h_a(v)$ and $y_a=h_a(w)$ gives
\begin{align*}
|\mathcal T_Dv(s)-\mathcal T_Dw(s)|
&\le
\max_{a\in D(s)}
\left|
\gamma\sum_{s'}P(s'\mid s,a)
[v(s')-w(s')]
\right|\\
&\le
\gamma\max_{a\in D(s)}
\sum_{s'}P(s'\mid s,a)
|v(s')-w(s')|\\
&\le
\gamma\|v-w\|_\infty
\max_{a\in D(s)}
\sum_{s'}P(s'\mid s,a)\\
&=\gamma\|v-w\|_\infty.
\end{align*}
Taking the maximum over $s$ proves
\[
\|\mathcal T_Dv-\mathcal T_Dw\|_\infty
\le\gamma\|v-w\|_\infty.
\]
Because $\gamma<1$, the Banach fixed-point theorem gives a unique fixed point, denoted $V_D^*$.

For each state $s$, the set $D(s)$ is finite and nonempty, so there exists
\[
a_D^*(s)\in
\argmax_{a\in D(s)}
\left[
r(s,a)+\gamma\sum_{s'}P(s'\mid s,a)V_D^*(s')
\right].
\]
Let $\pi_D^*$ select $a_D^*(s)$ with probability one. Define the policy Bellman operator
\[
(\mathcal T^\pi v)(s)
=
\sum_a\pi(a\mid s)
\left[
r(s,a)+\gamma\sum_{s'}P(s'\mid s,a)v(s')
\right].
\]
By the choice of $a_D^*(s)$,
\[
\mathcal T^{\pi_D^*}V_D^*
=\mathcal T_DV_D^*
=V_D^*.
\]
The operator $\mathcal T^{\pi_D^*}$ is a $\gamma$ contraction by the same calculation as above, and its unique fixed point is $V^{\pi_D^*}$. Hence
\[
V^{\pi_D^*}=V_D^*.
\]

Now fix any $\pi\in\Pi(D)$. For every $v$ and every state,
\begin{align*}
(\mathcal T^\pi v)(s)
&=\sum_{a\in D(s)}\pi(a\mid s)
\left[
r(s,a)+\gamma\sum_{s'}P(s'\mid s,a)v(s')
\right]\\
&\le
\max_{a\in D(s)}
\left[
r(s,a)+\gamma\sum_{s'}P(s'\mid s,a)v(s')
\right]\\
&=(\mathcal T_Dv)(s).
\end{align*}
Both operators are monotone. Starting from the zero function and applying the preceding inequality inductively yields
\[
(\mathcal T^\pi)^n0
\le
\mathcal T_D^n0
\qquad\text{for every }n.
\]
The left side converges to $V^\pi$ and the right side converges to $V_D^*$ because both operators are contractions. Therefore
\[
V^\pi\le V_D^*
\]
componentwise. This proves that $\pi_D^*$ is optimal over $\Pi(D)$.

For $\pi\in\Pi(D)$, define
\[
\rho^\pi(s,a)
=(1-\gamma)\sum_{t=0}^{\infty}\gamma^t
\Pr^\pi(s_t=s,a_t=a\mid s_0\sim\mu_0).
\]
Every component is nonnegative. If $a\notin D(s)$, then
$\pi(a\mid s)=0$, and hence
$\Pr^\pi(s_t=s,a_t=a)=0$ for every $t$; therefore
$\rho^\pi(s,a)=0$.

Let
\[
d^\pi(s)=\sum_a\rho^\pi(s,a).
\]
Using the law of total probability and the Markov property,
\begin{align*}
d^\pi(s)
&=(1-\gamma)\sum_{t=0}^{\infty}\gamma^t
\Pr^\pi(s_t=s)\\
&=(1-\gamma)\mu_0(s)
 +(1-\gamma)\sum_{t=1}^{\infty}\gamma^t
 \Pr^\pi(s_t=s)\\
&=(1-\gamma)\mu_0(s)
 +\gamma(1-\gamma)\sum_{j=0}^{\infty}\gamma^j
 \Pr^\pi(s_{j+1}=s)\\
&=(1-\gamma)\mu_0(s)\\
&\quad+\gamma(1-\gamma)\sum_{j=0}^{\infty}\gamma^j
 \sum_{s',a'}
 \Pr^\pi(s_j=s',a_j=a')P(s\mid s',a')\\
&=(1-\gamma)\mu_0(s)
 +\gamma\sum_{s',a'}P(s\mid s',a')\rho^\pi(s',a').
\end{align*}
Thus $\rho^\pi\in\mathcal D(D)$.

Let $\rho\in\mathcal D(D)$ and define
\[
d(s)=\sum_a\rho(s,a).
\]
If $d(s)>0$, set
\[
\pi_\rho(a\mid s)=\frac{\rho(s,a)}{d(s)}.
\]
If $d(s)=0$, let $\pi_\rho(\cdot\mid s)$ place probability one on the first action in $D(s)$ under the fixed global action order. This action exists because $D(s)\ne\emptyset$. The confinement constraints imply that $\pi_\rho(a\mid s)=0$ for every $a\notin D(s)$, so
$\pi_\rho\in\Pi(D)$.

Define
\[
P^{\pi_\rho}(s'\mid s)
=\sum_a\pi_\rho(a\mid s)P(s'\mid s,a).
\]
We claim that the flow constraints for $\rho$ can be written as
\[
d=(1-\gamma)\mu_0
+\gamma(P^{\pi_\rho})^\top d.
\]
Indeed, if $d(s')>0$, then
$d(s')\pi_\rho(a\mid s')=\rho(s',a)$.
If $d(s')=0$, nonnegativity and
$\sum_a\rho(s',a)=0$ imply
$\rho(s',a)=0=d(s')\pi_\rho(a\mid s')$ for every action.
Consequently,
\begin{align*}
\sum_{s',a'}P(s\mid s',a')\rho(s',a')
&=\sum_{s'}d(s')
\sum_{a'}\pi_\rho(a'\mid s')P(s\mid s',a')\\
&=\sum_{s'}P^{\pi_\rho}(s\mid s')d(s'),
\end{align*}
which proves the vector equation.

Because $P^{\pi_\rho}$ is row stochastic,
$\|(P^{\pi_\rho})^\top x\|_1\le\|x\|_1$ for every vector $x$.
Hence
\[
\|\gamma(P^{\pi_\rho})^\top\|_{1\to1}\le\gamma<1,
\]
and the Neumann series
\[
\left[I-\gamma(P^{\pi_\rho})^\top\right]^{-1}
=
\sum_{t=0}^{\infty}
\gamma^t\left[(P^{\pi_\rho})^\top\right]^t
\]
converges. Thus the state-flow equation has the unique solution
\[
d=(1-\gamma)
\sum_{t=0}^{\infty}
\gamma^t\left[(P^{\pi_\rho})^\top\right]^t\mu_0,
\]
which is exactly the normalized discounted state occupancy of $\pi_\rho$.
Multiplying this state occupancy by $\pi_\rho(a\mid s)$ shows that the state-action occupancy of $\pi_\rho$ is
\[
d(s)\pi_\rho(a\mid s)=\rho(s,a)
\]
at every state and action. Therefore every point in $\mathcal D(D)$ is generated by a policy in $\Pi(D)$.

For any $\pi\in\Pi(D)$,
\begin{align*}
\sum_{s,a}\rho^\pi(s,a)r(s,a)
&=(1-\gamma)\sum_{t=0}^{\infty}\gamma^t
\sum_{s,a}
\Pr^\pi(s_t=s,a_t=a)r(s,a)\\
&=(1-\gamma)\sum_{t=0}^{\infty}\gamma^t
\mathbb E^\pi[r(s_t,a_t)]\\
&=(1-\gamma)\sum_{t=0}^{\infty}\gamma^t
\mathbb E^\pi[R_t]\\
&=(1-\gamma)V^\pi(\mu_0).
\end{align*}
The third equality uses
$\mathbb E[R_t\mid s_t,a_t]=r(s_t,a_t)$ and the tower property.
Steps 3 and 4 give a correspondence between supported policies and feasible occupancies, so maximizing the LP is equivalent to maximizing
$V^\pi(\mu_0)$ over $\Pi(D)$. Its optimal value is therefore
$(1-\gamma)V_D^*(\mu_0)$.

Summing the flow constraints over all states yields
\begin{align*}
\sum_{s,a}\rho(s,a)
&=(1-\gamma)\sum_s\mu_0(s)\\
&\quad+\gamma
\sum_{s',a'}\rho(s',a')
\sum_sP(s\mid s',a')\\
&=(1-\gamma)+\gamma\sum_{s',a'}\rho(s',a').
\end{align*}
Therefore
\[
\sum_{s,a}\rho(s,a)=1.
\]
The set $\mathcal D(D)$ is defined by finitely many linear equalities, linear zero constraints, and nonnegativity constraints, so it is closed. The preceding identity and nonnegativity make it bounded. It is nonempty because any deterministic policy supported on $D$ has a feasible occupancy. Hence $\mathcal D(D)$ is compact, and the continuous linear objective attains its maximum.

It is enough to regard $\rho(s,a)$ for $a\in D(s)$ as the primal variables. Associate an unrestricted dual variable $v(s)$ with each flow equality. The coefficient of a primal variable $\rho(x,a)$ in the equality indexed by $s$ is
\[
\Ind\{s=x\}-\gamma P(s\mid x,a).
\]
The corresponding dual inequality is
\[
\sum_s v(s)
[\Ind\{s=x\}-\gamma P(s\mid x,a)]
\ge r(x,a),
\]
which is equivalent to
\[
v(x)\ge r(x,a)
+\gamma\sum_sP(s\mid x,a)v(s).
\]
This gives the dual in~\eqref{eq:restricted_occupancy_dual}. The primal is feasible and bounded. The dual is feasible: the constant vector
\[
v(s)=\frac1{1-\gamma}
\]
satisfies, for every admissible action,
\[
r(s,a)+\gamma\sum_{s'}P(s'\mid s,a)v(s')
\le1+\frac{\gamma}{1-\gamma}
=\frac1{1-\gamma}
=v(s).
\]
The finite-dimensional linear-programming strong-duality theorem therefore gives equality of the primal and dual optimal values.

Let $\mathcal O_D$ be the set of primal optimizers. It is a nonempty closed subset of the compact set $\mathcal D(D)$ and is therefore compact.
Order the $n=|\mathcal S||\mathcal A|$ coordinates as
$1,\ldots,n$. Set $\mathcal O_D^{(0)}=\mathcal O_D$. For
$j=1,\ldots,n$, define
\[
m_j=\min_{\rho\in\mathcal O_D^{(j-1)}}\rho_j,
\qquad
\mathcal O_D^{(j)}
=\{\rho\in\mathcal O_D^{(j-1)}:\rho_j=m_j\}.
\]
Each minimum is attained because
$\mathcal O_D^{(j-1)}$ is nonempty and compact, so each
$\mathcal O_D^{(j)}$ is again nonempty and compact. At the last step all coordinates are fixed, so
$\mathcal O_D^{(n)}$ contains exactly one vector. This vector is the lexicographically smallest optimal occupancy and is denoted
$\rho_D^{\mathrm{lex}}$.

Finally, reconstruct a policy from $\rho_D^{\mathrm{lex}}$, using the first action in $D(s)$ at every state with zero occupancy. The reconstructed policy has occupancy
$\rho_D^{\mathrm{lex}}$ and attains the restricted optimal value. This supplies the unique response convention needed to evaluate the outer objective.
\end{proof}

\begin{proof}[Proof of Proposition~\ref{prop:outer_existence}]
Fix $\theta$. Because $\mathcal S$ and $\mathcal A$ are finite, the collection of score values
\[
\{g_\theta(s,a):s\in\mathcal S, a\in\mathcal A\}
\]
is finite. As $\tau$ varies, a raw set changes only when $\tau$ crosses one of these values. The fallback action is also fixed by $\theta$ and the tie-breaking order. Hence only finitely many complete raw--operational designs
$G=(C_{\theta,\tau},D_{\theta,\tau})$ can be induced by thresholds.

For any induced design, the lower occupancy LP depends only on $D$. The lexicographic construction in Theorem~\ref{thm:inner_wellposed} therefore fixes one optimal occupancy and makes both the objective and coverage constraint functions of $G$ alone. The feasible induced designs form a finite set, so the objective attains its minimum whenever this set is nonempty.

It remains to prove nonemptiness. Since every score is finite, choose
\[
\tau_{\mathrm{all}}
\ge\max_{s,a}g_\theta(s,a).
\]
Then $C_{\theta,\tau_{\mathrm{all}}}(s)=D_{\theta,\tau_{\mathrm{all}}}(s)=\mathcal A$ for every state. The oracle target is nonempty, so the retention probability equals one. Thus $\tau_{\mathrm{all}}$ is feasible, and a minimizing threshold exists.
\end{proof}

\subsection{Proof of the Quantile Characterization}
The argument has two parts. The margin equivalence identifies the feasible thresholds, and nesting of the execution sets establishes size minimality under the same fixed occupancy.

\begin{proof}[Proof of Proposition~\ref{prop:quantile}]
Fix $\theta$ and $\pi$. For every state $s$,
\begin{align*}
&C_{\theta,\tau}(s)\cap\mathcal A^*_\varepsilon(s)
\ne\emptyset\\
&\quad\Longleftrightarrow
\text{there exists }a\in\mathcal A^*_\varepsilon(s)
\text{ with }g_\theta(s,a)\le\tau\\
&\quad\Longleftrightarrow
\min_{a\in\mathcal A^*_\varepsilon(s)}
g_\theta(s,a)\le\tau\\
&\quad\Longleftrightarrow
m_\theta(s)\le\tau.
\end{align*}
The minimum is well-defined because
$\mathcal A^*_\varepsilon(s)$ is finite and nonempty. Taking probability under
$s\sim d^\pi$ gives
\[
\Pr_{s\sim d^\pi}
\left(
C_{\theta,\tau}(s)\cap
\mathcal A^*_\varepsilon(s)\ne\emptyset
\right)
=
F_{\theta,\pi}(\tau).
\]
Thus $\tau$ satisfies the coverage requirement if and only if
\[
F_{\theta,\pi}(\tau)\ge1-\alpha.
\]

Let
\[
q=F^{-1}_{\theta,\pi}(1-\alpha)
=\inf\{x:F_{\theta,\pi}(x)\ge1-\alpha\}.
\]
Because $\mathcal S$ is finite, the random variable
$m_\theta(s)$ has finite support under $d^\pi$. Hence its CDF is a right-continuous step function and the set
\[
\{x:F_{\theta,\pi}(x)\ge1-\alpha\}
\]
has a smallest element. That element is $q$, so
$F_{\theta,\pi}(q)\ge1-\alpha$. If $\tau<q$, then by the definition of the infimum and the fact that $q$ is the smallest feasible support point,
\[
F_{\theta,\pi}(\tau)<1-\alpha.
\]
If $\tau\ge q$, monotonicity of the CDF gives
\[
F_{\theta,\pi}(\tau)\ge
F_{\theta,\pi}(q)\ge1-\alpha.
\]
Therefore the feasible thresholds are exactly the interval
$[q,\infty)$, and the smallest feasible threshold is
$q=\tau_{\min}(\theta,\pi)$.

It remains to prove the set-size minimality under the same fixed policy. Let
$\tau_1\le\tau_2$. The raw sets are nested:
\[
C_{\theta,\tau_1}(s)
\subseteq C_{\theta,\tau_2}(s)
\qquad\text{for every }s.
\]
We now verify that the operational sets are also nested.

If $C_{\theta,\tau_1}(s)\ne\emptyset$, then both operational sets equal their raw sets and
\[
\overline C_{\theta,\tau_1}(s)
=C_{\theta,\tau_1}(s)
\subseteq C_{\theta,\tau_2}(s)
=\overline C_{\theta,\tau_2}(s).
\]
If $C_{\theta,\tau_1}(s)=\emptyset$ and
$C_{\theta,\tau_2}(s)=\emptyset$, then both operational sets equal the same singleton
$\{a_\theta^{\min}(s)\}$.
Finally, suppose
$C_{\theta,\tau_1}(s)=\emptyset$ but
$C_{\theta,\tau_2}(s)\ne\emptyset$.
Nonemptiness of the latter implies
\[
\tau_2\ge\min_{a}g_\theta(s,a)
=g_\theta(s,a_\theta^{\min}(s)),
\]
so
$a_\theta^{\min}(s)\in C_{\theta,\tau_2}(s)$. Hence
\[
\overline C_{\theta,\tau_1}(s)
=\{a_\theta^{\min}(s)\}
\subseteq
C_{\theta,\tau_2}(s)
=\overline C_{\theta,\tau_2}(s).
\]
Thus
\[
|\overline C_{\theta,\tau_1}(s)|
\le
|\overline C_{\theta,\tau_2}(s)|
\]
for every state. Averaging with respect to the fixed distribution
$d^\pi$ preserves the inequality. Since every coverage-valid threshold is at least
$\tau_{\min}(\theta,\pi)$, that threshold minimizes the expected operational-set size among all thresholds evaluated under the same fixed policy.
\end{proof}

\subsection{Proof of Pathwise Calibration}
The update cannot move arbitrarily far beyond the realized margin range: outside that range, its direction points back toward it. We first formalize this interval bound and then express the cumulative calibration error in terms of threshold increments.

\begin{proof}[Proof of Theorem~\ref{thm:pathwise}]
Fix an arbitrary horizon $T\ge1$ and define
\begin{equation*}
\begin{aligned}
\underline m_T
&=\min\{\tau_1,m_1,\ldots,m_T\},\\
\overline m_T
&=\max\{\tau_1,m_1,\ldots,m_T\}.
\end{aligned}
\end{equation*}
Then $W_T=\overline m_T-\underline m_T$. Define
\[
\tau_T^-=\underline m_T-\alpha\eta_1,
\qquad
\tau_T^+=\overline m_T+(1-\alpha)\eta_1.
\]
We first prove by induction that
\begin{equation}
\tau_t\in[\tau_T^-,\tau_T^+]
\qquad
\text{for every }1\le t\le T+1.
\label{eq:threshold_invariant}
\end{equation}
The base case holds because
$\tau_1\in[\underline m_T,\overline m_T]$.
Assume the claim holds for some $t\le T$. Since
$\eta_t\le\eta_1$, there are three cases.

If $\tau_t<\underline m_T$, then
$m_t\ge\underline m_T>\tau_t$, so $e_t=1$ and
\[
\tau_{t+1}=\tau_t+(1-\alpha)\eta_t.
\]
The update is nonnegative, so
$\tau_{t+1}\ge\tau_t\ge\tau_T^-$. Also,
\begin{align*}
\tau_{t+1}
&<\underline m_T+(1-\alpha)\eta_t\\
&\le\overline m_T+(1-\alpha)\eta_1
=\tau_T^+.
\end{align*}

If
$\underline m_T\le\tau_t\le\overline m_T$,
then $e_t\in\{0,1\}$ and
\[
\tau_t-\alpha\eta_t
\le\tau_{t+1}
\le\tau_t+(1-\alpha)\eta_t.
\]
Using $\eta_t\le\eta_1$ and the bounds on $\tau_t$ yields
$\tau_T^-\le\tau_{t+1}\le\tau_T^+$.

If $\tau_t>\overline m_T$, then
$m_t\le\overline m_T<\tau_t$, so $e_t=0$ and
\[
\tau_{t+1}=\tau_t-\alpha\eta_t.
\]
The update is nonpositive, so
$\tau_{t+1}\le\tau_t\le\tau_T^+$. Moreover,
\begin{align*}
\tau_{t+1}
&>\overline m_T-\alpha\eta_t\\
&\ge\underline m_T-\alpha\eta_1
=\tau_T^-.
\end{align*}
This completes the induction. The invariant interval has width
\begin{equation}
\begin{aligned}
\tau_T^+-\tau_T^-
&=(\overline m_T-\underline m_T)+\eta_1\\
&=W_T+\eta_1.
\end{aligned}
\label{eq:invariant_width}
\end{equation}

The interval bound now controls the accumulated misses. Rearranging~\eqref{eq:online_threshold_update} gives
\[
e_t-\alpha
=\frac{\tau_{t+1}-\tau_t}{\eta_t}.
\]
Let $c_t=1/\eta_t$. Because $\eta_t$ is positive and nonincreasing, $c_t$ is positive and nondecreasing. Therefore
\begin{align*}
S_T
&=\sum_{t=1}^T(e_t-\alpha)\\
&=\sum_{t=1}^Tc_t(\tau_{t+1}-\tau_t)\\
&=c_T\tau_{T+1}-c_1\tau_1
-\sum_{t=2}^T(c_t-c_{t-1})\tau_t.
\end{align*}
Define
\[
\omega_1=\frac{c_1}{c_T},
\qquad
\omega_t=\frac{c_t-c_{t-1}}{c_T}
\quad(2\le t\le T).
\]
Each $\omega_t\ge0$, and
\[
\sum_{t=1}^T\omega_t
=\frac{c_1+\sum_{t=2}^T(c_t-c_{t-1})}{c_T}
=1.
\]
Hence
\[
S_T
=c_T\left(
\tau_{T+1}-\sum_{t=1}^T\omega_t\tau_t
\right).
\]
The weighted sum is a convex combination of
$\tau_1,\ldots,\tau_T$. By~\eqref{eq:threshold_invariant}, both this convex combination and $\tau_{T+1}$ lie in
$[\tau_T^-,\tau_T^+]$. Thus, using~\eqref{eq:invariant_width},
\[
|S_T|
\le c_T(\tau_T^+-\tau_T^-)
=\frac{W_T+\eta_1}{\eta_T}.
\]
Dividing by $T$ proves~\eqref{eq:pathwise_calibration_bound}. Since the argument used only the update and the realized margin interval, the bound holds for an arbitrary adaptive sequence. The convergence condition and the constant step consequence follow directly from its right hand side.
\end{proof}

\subsection{Proofs for the Value Results}
We begin with the counterexample separating retention from selection, then derive the exact value identity. The subsequent proofs bound its filtering and selection terms in turn, before combining them along an adaptive finite session.

\begin{proof}[Proof of Proposition~\ref{prop:coverage_insufficient}]
Consider one state $s$, two actions $a_g,a_b$, deterministic self-transitions, and rewards
\[
r(s,a_g)=1,
\qquad
r(s,a_b)=0.
\]
The optimal policy always chooses $a_g$, so
\[
V^*(s)=\sum_{t=0}^{\infty}\gamma^t
=\frac1{1-\gamma}.
\]
The optimal action values are
\begin{equation*}
\begin{aligned}
Q^*(s,a_g)
&=1+\gamma V^*(s)=\frac1{1-\gamma},\\
Q^*(s,a_b)
&=\gamma V^*(s)=\frac{\gamma}{1-\gamma}.
\end{aligned}
\end{equation*}
Therefore
$\mathcal A^*_0(s)=\{a_g\}$.
Let $D(s)=\{a_g,a_b\}$, so the set contains an optimal action with probability one. Let $\pi_b$ always choose $a_b$. This policy belongs to $\Pi(D)$ and receives zero reward at every time, hence
$V^{\pi_b}(s)=0$. Its value gap is
\[
V^*(s)-V^{\pi_b}(s)=\frac1{1-\gamma}.
\]
Lemma~\ref{lem:value_bounds} shows that no value gap can be larger, so the gap is maximal.
\end{proof}

\begin{proof}[Proof of Theorem~\ref{thm:value_decomposition}]
We first verify that both terms are losses, then express their sum as the optimal Bellman residual of the supported policy. Since $D(s)$ is nonempty and
$V^*(s)=\max_{a\in\mathcal A}Q^*(s,a)$,
\[
\max_{a\in D(s)}Q^*(s,a)\le V^*(s),
\]
so $\Delta_D(s)\ge0$. Also, an average cannot exceed the maximum of the averaged numbers:
\[
\sum_{a\in D(s)}\pi(a\mid s)Q^*(s,a)
\le
\max_{a\in D(s)}Q^*(s,a),
\]
so $\zeta_{\pi,D}(s)\ge0$.

Define the policy-averaged reward and transition kernel
\begin{equation*}
\begin{aligned}
r^\pi(s)
&=\sum_a\pi(a\mid s)r(s,a),\\
P^\pi(s'\mid s)
&=\sum_a\pi(a\mid s)P(s'\mid s,a).
\end{aligned}
\end{equation*}
For any function $f$ on $\mathcal S$, let
\[
(P^\pi f)(s)
=\sum_{s'}P^\pi(s'\mid s)f(s').
\]
To relate these local gaps to the total value loss, define the optimal Bellman residual of $\pi$ by
\[
b_\pi(s)
=V^*(s)-\sum_a\pi(a\mid s)Q^*(s,a).
\]
Using~\eqref{eq:optimal_bellman_appendix},
\begin{align*}
&\sum_a\pi(a\mid s)Q^*(s,a)\\
&\quad=\sum_a\pi(a\mid s)r(s,a)\\
&\qquad+\gamma\sum_{s'}
\left[\sum_a\pi(a\mid s)P(s'\mid s,a)\right]
V^*(s')\\
&\quad=r^\pi(s)+\gamma(P^\pi V^*)(s).
\end{align*}
Lemma~\ref{lem:bellman} gives
$V^\pi=r^\pi+\gamma P^\pi V^\pi$.
Subtracting this identity from $V^*$ yields, state by state,
\begin{align*}
V^*-V^\pi
&=V^*-r^\pi-\gamma P^\pi V^\pi\\
&=\left[V^*-r^\pi-\gamma P^\pi V^*\right]
+\gamma P^\pi(V^*-V^\pi)\\
&=b_\pi+\gamma P^\pi(V^*-V^\pi).
\end{align*}
Iterating this equality $N$ times gives
\begin{align*}
V^*-V^\pi
={}&\sum_{t=0}^{N-1}\gamma^t(P^\pi)^tb_\pi\\
&+\gamma^N(P^\pi)^N(V^*-V^\pi).
\end{align*}
A Markov operator is nonexpansive in the sup norm. Lemma~\ref{lem:value_bounds} therefore gives
\[
\left\|
\gamma^N(P^\pi)^N(V^*-V^\pi)
\right\|_\infty
\le\frac{\gamma^N}{1-\gamma}
\longrightarrow0.
\]
Consequently,
\[
V^*-V^\pi
=\sum_{t=0}^{\infty}\gamma^t(P^\pi)^tb_\pi.
\]
Average both sides over $s_0\sim\mu_0$. The definition~\eqref{eq:state_occupancy} implies
\begin{align*}
V^*(\mu_0)-V^\pi(\mu_0)
&=\sum_{t=0}^{\infty}\gamma^t
\mathbb E^\pi[b_\pi(s_t)]\\
&=\frac1{1-\gamma}
\mathbb E_{s\sim d^\pi}[b_\pi(s)].
\end{align*}
Because $\pi\in\Pi(D)$,
\begin{align*}
b_\pi(s)
&=V^*(s)-\sum_{a\in D(s)}
\pi(a\mid s)Q^*(s,a)\\
&=V^*(s)-\max_{a\in D(s)}Q^*(s,a)\\
&\quad+\max_{a\in D(s)}Q^*(s,a)
-\sum_{a\in D(s)}\pi(a\mid s)Q^*(s,a)\\
&=\Delta_D(s)+\zeta_{\pi,D}(s).
\end{align*}
Substitution proves~\eqref{eq:exact_value_decomposition}. The decomposition is exact because it splits the same Bellman residual, rather than bounding filtering and selection through separate trajectory comparisons.
\end{proof}

\begin{proof}[Proof of Corollary~\ref{cor:coverage_to_value}]
Define the hit set
\[
H
=\{s:C(s)\cap\mathcal A^*_\varepsilon(s)\ne\emptyset\}
\]
and its miss probability under the policy occupancy:
\[
\alpha_\pi
=\Pr_{s\sim d^\pi}(s\notin H).
\]
Assumption~\eqref{eq:generic_coverage_assumption} gives
\[
\alpha_\pi\le\alpha.
\]

Fix a state $s\in H$. By definition of $H$, there exists at least one action
\[
a_H(s)\in C(s)\cap\mathcal A^*_\varepsilon(s).
\]
Since $C(s)\subseteq D(s)$,
\[
a_H(s)\in D(s).
\]
Therefore
\begin{align*}
\max_{a\in D(s)}Q^*(s,a)
&\ge Q^*(s,a_H(s))\\
&\ge V^*(s)-\varepsilon,
\end{align*}
where the second inequality follows from
$a_H(s)\in\mathcal A^*_\varepsilon(s)$. Rearranging gives
\[
\Delta_D(s)
=V^*(s)-\max_{a\in D(s)}Q^*(s,a)
\le\varepsilon.
\]
By definition of $\Delta_{\max}$, every nonempty set $D(s)$ also satisfies
\[
0\le\Delta_D(s)\le\Delta_{\max}.
\]
Hence, on the hit set,
\[
\Delta_D(s)\le
\min\{\varepsilon,\Delta_{\max}\}
=\varepsilon_0,
\]
while on the miss set,
\[
\Delta_D(s)\le\Delta_{\max}.
\]

Split the occupancy expectation over $H$ and its complement:
\begin{align*}
\mathbb E_{s\sim d^\pi}[\Delta_D(s)]
&=
\mathbb E[\Delta_D(s)\Ind\{s\in H\}]\\
&\quad+
\mathbb E[\Delta_D(s)\Ind\{s\notin H\}]\\
&\le
\varepsilon_0\Pr(s\in H)
+\Delta_{\max}\Pr(s\notin H)\\
&=(1-\alpha_\pi)\varepsilon_0
+\alpha_\pi\Delta_{\max}\\
&=\varepsilon_0
+\alpha_\pi(\Delta_{\max}-\varepsilon_0).
\end{align*}
Since $\Delta_{\max}-\varepsilon_0\ge0$ and
$\alpha_\pi\le\alpha$,
\begin{align*}
\mathbb E_{d^\pi}[\Delta_D(s)]
&\le
\varepsilon_0
+\alpha(\Delta_{\max}-\varepsilon_0)\\
&=(1-\alpha)\varepsilon_0
+\alpha\Delta_{\max}.
\end{align*}
The exact decomposition in Theorem~\ref{thm:value_decomposition} gives
\[
V^*(\mu_0)-V^\pi(\mu_0)
=
\frac{
\mathbb E_{d^\pi}[\Delta_D(s)]
+\overline\zeta_{\pi,D}
}{1-\gamma}.
\]
Substituting the preceding bound proves~\eqref{eq:coverage_to_value_bound}.
\end{proof}

\begin{proof}[Proof of Corollary~\ref{cor:bilevel_value}]
Feasibility of $(\theta,\tau)$ in
Problem~\eqref{eq:bilevel_objective} gives
\[
\Pr_{s\sim d^{\pi^*_{\theta,\tau}}}
\left(
C_{\theta,\tau}(s)
\cap\mathcal A^*_\varepsilon(s)
\ne\emptyset
\right)
\ge1-\alpha.
\]
By construction,
\[
C_{\theta,\tau}(s)\subseteq
D_{\theta,\tau}(s)
\]
and $D_{\theta,\tau}(s)$ is nonempty for every state.
The lower-level policy satisfies
\[
\pi^*_{\theta,\tau}\in
\Pi(D_{\theta,\tau}).
\]
Thus every assumption of Corollary~\ref{cor:coverage_to_value} holds after setting
\[
C=C_{\theta,\tau},
\qquad
D=D_{\theta,\tau},
\qquad
\pi=\pi^*_{\theta,\tau}.
\]
The stated bound follows directly.
\end{proof}

The occupancy bound concerns retention averaged under a specified policy. Under the stronger condition that every state retains a target action, we can instead compare the restricted and unrestricted optimal values at every initial state.

\begin{corollary}[Pointwise retention]
\label{cor:pointwise_retention}
If $D(s)\cap\mathcal A^*_\varepsilon(s)\ne\emptyset$ for every state, then
\begin{equation}
0\le V^*(s)-V_D^*(s)\le\frac{\varepsilon_0}{1-\gamma}
\qquad\forall s.
\label{eq:pointwise_retention_bound}
\end{equation}
\end{corollary}

\begin{proof}[Proof of Corollary~\ref{cor:pointwise_retention}]
For every state $s$, choose
\[
a_Q(s)\in\argmax_{a\in D(s)}Q^*(s,a)
\]
using a fixed tie-breaking rule, and let $\pi_Q$ be the deterministic policy that selects $a_Q(s)$. Because $a_Q(s)\in D(s)$,
\[
\pi_Q\in\Pi(D).
\]

The pointwise retention assumption gives an action
$a_H(s)\in D(s)\cap\mathcal A^*_\varepsilon(s)$ at every state. Therefore
\begin{align*}
\max_{a\in D(s)}Q^*(s,a)
&\ge Q^*(s,a_H(s))\\
&\ge V^*(s)-\varepsilon.
\end{align*}
Together with $\Delta_D(s)\le\Delta_{\max}$, this yields
\[
\Delta_D(s)
=V^*(s)-\max_{a\in D(s)}Q^*(s,a)
\le\varepsilon_0
\qquad\text{for every }s.
\]

For the policy $\pi_Q$, the optimal Bellman residual is
\begin{align*}
b_{\pi_Q}(s)
&=V^*(s)-\sum_a\pi_Q(a\mid s)Q^*(s,a)\\
&=V^*(s)-Q^*(s,a_Q(s))\\
&=\Delta_D(s)\\
&\le\varepsilon_0.
\end{align*}
The residual identity derived in the proof of
Theorem~\ref{thm:value_decomposition} gives
\[
V^*-V^{\pi_Q}
=b_{\pi_Q}
+\gamma P^{\pi_Q}(V^*-V^{\pi_Q}).
\]
Taking sup norms, using
$\|b_{\pi_Q}\|_\infty\le\varepsilon_0$, and using that a Markov operator is nonexpansive in the sup norm, we obtain
\[
\|V^*-V^{\pi_Q}\|_\infty
\le
\varepsilon_0
+\gamma\|V^*-V^{\pi_Q}\|_\infty.
\]
Rearranging gives
\[
\|V^*-V^{\pi_Q}\|_\infty
\le\frac{\varepsilon_0}{1-\gamma}.
\]

By Theorem~\ref{thm:inner_wellposed}, the restricted optimal value dominates every supported policy componentwise:
\[
V_D^*\ge V^{\pi_Q}.
\]
Also $V_D^*\le V^*$ because the restricted policy class is a subset of the unrestricted class. Hence, for every state,
\[
0\le V^*(s)-V_D^*(s)
\le V^*(s)-V^{\pi_Q}(s)
\le\frac{\varepsilon_0}{1-\gamma}.
\]
This proves~\eqref{eq:pointwise_retention_bound}.
\end{proof}

We now replace oracle retention by proxy retention. The key step is to compare any proxy target action with an oracle maximizer through the pairwise fidelity condition.

\begin{proof}[Proof of Proposition~\ref{prop:proxy_gap_transfer}]
Fix a state $s$, an action
$a\in\widetilde{\mathcal A}_\varepsilon(s)$, and
$a^*(s)\in\argmax_bQ^*(s,b)$. Proxy near-optimality implies
\begin{align*}
\widetilde Q(s,a^*(s))-\widetilde Q(s,a)
&\le \max_b\widetilde Q(s,b)-\widetilde Q(s,a)\\
&\le\varepsilon.
\end{align*}
Apply Assumption~\ref{ass:proxy_gap} to the pair $(a^*(s),a)$:
\begin{align*}
V^*(s)-Q^*(s,a)
&=Q^*(s,a^*(s))-Q^*(s,a)\\
&\le\lambda[\widetilde Q(s,a^*(s))-\widetilde Q(s,a)]
+\delta_{\mathrm{prox}}\\
&\le\lambda\varepsilon+\delta_{\mathrm{prox}}.
\end{align*}
Thus $a\in\mathcal A^*_{\lambda\varepsilon+\delta_{\mathrm{prox}}}(s)$, proving~\eqref{eq:proxy_set_inclusion}. If a set $C(s)$ intersects the proxy target, an action in that intersection also belongs to the enlarged oracle target. This proves the stated hit-event implication under any state distribution.
\end{proof}

\begin{proof}[Proof of Proposition~\ref{prop:immediate_proxy}]
The optimal Bellman identity gives, for every state and action pair,
\begin{align*}
&Q^*(s,a)-Q^*(s,b)-[r(s,a)-r(s,b)]\\
&\quad=\gamma\sum_{s'}
[P(s'\mid s,a)-P(s'\mid s,b)]V^*(s').
\end{align*}
Taking absolute values and then the maximum proves the pairwise condition with $\lambda=1$ and $\delta_{\mathrm{prox}}=\delta_{\mathrm{dyn}}$.

For a learned response proxy $\widehat p$, add and subtract the immediate-reward difference:
\begin{align*}
&\left|Q^*(s,a)-Q^*(s,b)
-[\widehat p(s,a)-\widehat p(s,b)]\right|\\
&\quad\le
\left|Q^*(s,a)-Q^*(s,b)-[r(s,a)-r(s,b)]\right|\\
&\qquad+|r(s,a)-\widehat p(s,a)|
+|r(s,b)-\widehat p(s,b)|\\
&\quad\le\delta_{\mathrm{dyn}}+2\delta_{\mathrm{resp}}.
\end{align*}
This proves the stated extension.

By Lemma~\ref{lem:value_bounds}, the range of $V^*$ is at most $(1-\gamma)^{-1}$. For any two probability distributions $P_1,P_2$ and any function $f$ with range at most $B$,
\[
|\mathbb E_{P_1}f-\mathbb E_{P_2}f|
\le B\,\TV(P_1,P_2).
\]
Apply this inequality with $f=V^*$ and $B=(1-\gamma)^{-1}$, then maximize over $(s,a,b)$ and multiply by $\gamma$ to obtain~\eqref{eq:continuation_tv_bound}.
\end{proof}

Proxy fidelity controls the filtering term. The next consequence controls the selection term by transferring an approximate critic maximization guarantee to the true optimal action values.

\begin{corollary}[Approximate within-set selection]
\label{cor:selector_transfer}
Assume $D(s)\ne\emptyset$ for every state and let $\pi\in\Pi(D)$. Suppose $\|\widehat Q-Q^*\|_\infty\le\delta_{\mathrm{crit}}$ and the within-set condition~\eqref{eq:estimated_selection_error} holds. Define
$\overline\delta_{\mathrm{sel}}=\mathbb E_{s\sim d^\pi}[\delta_{\mathrm{sel}}(s)]$.
Then
\begin{equation}
\zeta_{\pi,D}(s)\le2\delta_{\mathrm{crit}}+\delta_{\mathrm{sel}}(s).
\label{eq:true_selection_error_bound}
\end{equation}
If~\eqref{eq:generic_coverage_assumption} also holds, then
\begin{equation}
V^*(\mu_0)-V^\pi(\mu_0)
\le\frac{(1-\alpha)\varepsilon_0+\alpha\Delta_{\max}
+2\delta_{\mathrm{crit}}+\overline\delta_{\mathrm{sel}}}{1-\gamma}.
\label{eq:critic_selector_value_bound}
\end{equation}
\end{corollary}

\begin{proof}[Proof of Corollary~\ref{cor:selector_transfer}]
Fix a state $s$ and abbreviate $D_s=D(s)$.
The uniform critic error implies, for every $a$,
\[
Q^*(s,a)\le\widehat Q(s,a)+\delta_{\mathrm{crit}}
\]
and
\[
Q^*(s,a)\ge\widehat Q(s,a)-\delta_{\mathrm{crit}}.
\]
Taking a maximum in the first inequality gives
\[
\max_{a\in D_s}Q^*(s,a)
\le
\max_{a\in D_s}\widehat Q(s,a)
+\delta_{\mathrm{crit}}.
\]
Averaging the second inequality under
$\pi(\cdot\mid s)$ gives
\begin{align*}
\sum_{a\in D_s}\pi(a\mid s)Q^*(s,a)
&\ge
\sum_{a\in D_s}\pi(a\mid s)
[\widehat Q(s,a)-\delta_{\mathrm{crit}}]\\
&=
\sum_{a\in D_s}\pi(a\mid s)\widehat Q(s,a)
-\delta_{\mathrm{crit}},
\end{align*}
because $\sum_{a\in D_s}\pi(a\mid s)=1$.
Subtracting the second bound from the first yields
\begin{align*}
\zeta_{\pi,D}(s)
&=
\max_{a\in D_s}Q^*(s,a)
-\sum_{a\in D_s}\pi(a\mid s)Q^*(s,a)\\
&\le
\max_{a\in D_s}\widehat Q(s,a)
-\sum_{a\in D_s}\pi(a\mid s)\widehat Q(s,a)
+2\delta_{\mathrm{crit}}\\
&\le
\delta_{\mathrm{sel}}(s)
+2\delta_{\mathrm{crit}}.
\end{align*}
This proves~\eqref{eq:true_selection_error_bound}.

Average the statewise inequality under $d^\pi$:
\begin{align*}
\overline\zeta_{\pi,D}
&=
\mathbb E_{s\sim d^\pi}[\zeta_{\pi,D}(s)]\\
&\le
2\delta_{\mathrm{crit}}
+\mathbb E_{s\sim d^\pi}[\delta_{\mathrm{sel}}(s)]\\
&=
2\delta_{\mathrm{crit}}
+\overline\delta_{\mathrm{sel}}.
\end{align*}
Substituting this bound into
Corollary~\ref{cor:coverage_to_value} proves
\eqref{eq:critic_selector_value_bound}.
\end{proof}

The combined bound now follows by applying these two transfers in sequence: proxy hits give an enlarged oracle target, and critic and selector errors bound the remaining selection loss.

\begin{proof}[Proof of Corollary~\ref{cor:combined_approximation}]
Proposition~\ref{prop:proxy_gap_transfer} gives, for every state,
\[
\widetilde{\mathcal A}_\varepsilon(s)
\subseteq
\mathcal A^*_{\lambda\varepsilon+\delta_{\mathrm{prox}}}(s).
\]
Therefore the proxy coverage assumption
\eqref{eq:proxy_coverage_assumption} implies
\[
\Pr_{s\sim d^\pi}
\left(
C(s)\cap
\mathcal A^*_{\lambda\varepsilon+\delta_{\mathrm{prox}}}(s)
\ne\emptyset
\right)
\ge1-\alpha.
\]
Apply Corollary~\ref{cor:coverage_to_value} with the tolerance
\[
\varepsilon'
=\lambda\varepsilon+\delta_{\mathrm{prox}}.
\]
The truncated hit-state loss in that corollary is
\[
\min\{\varepsilon',\Delta_{\max}\}
=\varepsilon_{\mathrm{rec}}.
\]
Thus
\[
V^*(\mu_0)-V^\pi(\mu_0)
\le
\frac{
(1-\alpha)\varepsilon_{\mathrm{rec}}
+\alpha\Delta_{\max}
+\overline\zeta_{\pi,D}
}{1-\gamma}.
\]
Corollary~\ref{cor:selector_transfer} gives
\[
\overline\zeta_{\pi,D}
\le
2\delta_{\mathrm{crit}}
+\overline\delta_{\mathrm{sel}}.
\]
Substitution proves~\eqref{eq:combined_approximation_bound}.
\end{proof}

For patience, we use the same retention and selection distinction at a single interaction. The collection response is controlled by event inclusion, whereas the response of a selected item also includes the choice error.

\begin{proof}[Proof of Proposition~\ref{prop:patience_loss}]
Let
\[
H_p
=\left\{s:D(s)\text{ contains an action }a
\text{ with }p(s,a)\ge p^*(s)-\varepsilon\right\}.
\]
Let $\alpha_p=\Pr_{s\sim\nu}(s\notin H_p)\le\alpha$.

For every action $a\in D(s)$, the event $\{Y_a=1\}$ is contained in the union defining $p_D^{\mathrm{set}}(s)$. Hence
\[
p_D^{\mathrm{set}}(s)\ge\max_{a\in D(s)}p(s,a)
\]
without any independence assumption. On $H_p$, this gives
\[
p^*(s)-p_D^{\mathrm{set}}(s)\le\varepsilon_p.
\]
Outside $H_p$, the same difference is at most one. Therefore
\begin{align*}
\mathbb E_\nu[1-p_D^{\mathrm{set}}(s)]
&=\mathbb E_\nu[1-p^*(s)]
+\mathbb E_\nu[p^*(s)-p_D^{\mathrm{set}}(s)]\\
&\le\mathbb E_\nu[1-p^*(s)]
+(1-\alpha_p)\varepsilon_p+\alpha_p\\
&\le\mathbb E_\nu[1-p^*(s)]
+(1-\alpha)\varepsilon_p+\alpha,
\end{align*}
where the last inequality uses $\varepsilon_p\le1$.

For the pre-response item selector,
\begin{align*}
p^*(s)-p_{D,\kappa}^{\mathrm{item}}(s)
&=p^*(s)-\max_{a\in D(s)}p(s,a)
+\delta_{\mathrm{sel}}^p(s).
\end{align*}
The first term is at most $\varepsilon_p$ on $H_p$ and at most one outside $H_p$. Taking expectations gives
\[
\mathbb E_\nu[p^*(s)-p_{D,\kappa}^{\mathrm{item}}(s)]
\le
(1-\alpha)\varepsilon_p+\alpha
+\overline\delta_{\mathrm{sel}}^p,
\]
which proves~\eqref{eq:selected_no_response_bound}.

For either mechanism, denote its positive-response probability by $p_{\mathrm{mech}}(s)$. Its expected patience decrease is
\[
c_+p_{\mathrm{mech}}(s)+c_-[1-p_{\mathrm{mech}}(s)]
=c_++(c_--c_+)[1-p_{\mathrm{mech}}(s)].
\]
The best single item replaces $p_{\mathrm{mech}}(s)$ by $p^*(s)$. Subtracting and applying the corresponding no-response bound proves~\eqref{eq:patience_decrement_bound} and~\eqref{eq:selected_patience_decrement_bound}.
\end{proof}

\begin{proof}[Proof of Theorem~\ref{thm:finite_session}]
The proof combines a finite horizon performance identity with separate bounds for filtering and selection. Unlike the stationary argument, it conditions on the information available at each stage, so the policy and critic may change during the session.

By the finite-horizon Bellman equation and the controlled Markov property,
\begin{align*}
Q_h^*(s_h,a_h)
&=r_h(s_h,a_h)
+\mathbb E\!\left[
V_{h+1}^*(s_{h+1})
\mid\mathcal H_h,a_h
\right],\\
\mathbb E[R_h\mid\mathcal H_h,a_h]
&=r_h(s_h,a_h).
\end{align*}
The tower property therefore gives
\[
\mathbb E\!\left[
V_h^*(s_h)-Q_h^*(s_h,a_h)
\right]
=
\mathbb E\!\left[
V_h^*(s_h)-R_h-V_{h+1}^*(s_{h+1})
\right].
\]
Summing from $h=1$ to $H$, using $s_1\sim\mu_0$ and $V_{H+1}^*=0$, yields
\begin{equation}
V_1^*(\mu_0)-V_1^{\mathrm{RLCP}}(\mu_0)
=\sum_{h=1}^H
\mathbb E\!\left[
V_h^*(s_h)-Q_h^*(s_h,a_h)
\right].
\label{eq:episodic_performance_identity}
\end{equation}

For each stage, split the optimal action gap into the loss from the execution set and the loss from the selected action:
\begin{align*}
\Delta_h
&=V_h^*(s_h)-\max_{a\in D_h}Q_h^*(s_h,a),\\
\zeta_h
&=\max_{a\in D_h}Q_h^*(s_h,a)-Q_h^*(s_h,a_h).
\end{align*}
Because $D_h$ is nonempty and $a_h\in D_h$ almost surely, both gaps are nonnegative, $\Delta_h\le\Delta_H^{\max}$, and
\[
V_h^*(s_h)-Q_h^*(s_h,a_h)=\Delta_h+\zeta_h.
\]

Suppose first that $e_h^K=0$. Choose
$b_h\in C_h^K\cap\widetilde{\mathcal A}_{\varepsilon,h}(s_h)$ and
$a_h^*\in\argmax_{a\in\mathcal A}Q_h^*(s_h,a)$. Membership of $b_h$ in the proxy target implies
\[
\widetilde Q_h(s_h,a_h^*)-
\widetilde Q_h(s_h,b_h)
\le
\max_{a\in\mathcal A}\widetilde Q_h(s_h,a)
-\widetilde Q_h(s_h,b_h)
\le\varepsilon.
\]
Applying~\eqref{eq:finite_session_proxy_fidelity} to $(a_h^*,b_h)$ gives
\begin{align*}
V_h^*(s_h)-Q_h^*(s_h,b_h)
&=Q_h^*(s_h,a_h^*)-Q_h^*(s_h,b_h)\\
&\le\lambda\bigl(
\widetilde Q_h(s_h,a_h^*)-
\widetilde Q_h(s_h,b_h)
\bigr)+\delta_{\mathrm{prox}}\\
&\le\lambda\varepsilon+\delta_{\mathrm{prox}}.
\end{align*}
Since $b_h\in C_h^K\subseteq D_h$, this proves $\Delta_h\le\varepsilon_H$ on $\{e_h^K=0\}$. On $\{e_h^K=1\}$, the universal bound $\Delta_h\le\Delta_H^{\max}$ applies. Hence, pathwise,
\begin{equation}
\Delta_h
\le\varepsilon_H
+(\Delta_H^{\max}-\varepsilon_H)e_h^K.
\label{eq:finite_session_pointwise_filtering}
\end{equation}

By definition, a capped miss is either a raw miss or a raw hit removed by the cap, and these events are disjoint. Thus
\[
e_h^K=\widetilde e_h+h_h^K.
\]
The recursion~\eqref{eq:finite_session_threshold_update} has exactly the form required by Theorem~\ref{thm:pathwise}. Since the step sizes are positive and nonincreasing and $\widetilde W_H\le W$, that theorem gives, almost surely,
\[
\frac1H\sum_{h=1}^H\widetilde e_h
\le\alpha+\frac{W+\eta_1}{\eta_HH}.
\]
Consequently,
\[
\frac1H\sum_{h=1}^He_h^K
\le\alpha+\frac{W+\eta_1}{\eta_HH}
+\frac1H\sum_{h=1}^Hh_h^K.
\]
Average~\eqref{eq:finite_session_pointwise_filtering} over stages and then take expectations. This yields
\begin{equation}
\begin{aligned}
\frac1H\sum_{h=1}^H\mathbb E[\Delta_h]
\le{}&(1-\alpha)\varepsilon_H
+\alpha\Delta_H^{\max}\\
&+(\Delta_H^{\max}-\varepsilon_H)
\left(
\frac{W+\eta_1}{\eta_HH}
+\overline h_H^K
\right).
\end{aligned}
\label{eq:finite_session_filtering_bound}
\end{equation}

The remaining term is the gap between the best retained action and the action actually selected. Since $D_h$ and $\widehat Q_h$ are $\mathcal H_h$ measurable, the critic condition implies
\begin{align*}
\max_{a\in D_h}Q_h^*(s_h,a)
&\le\max_{a\in D_h}\widehat Q_h(s_h,a)
+\delta_{\mathrm{crit}},\\
\mathbb E\!\left[Q_h^*(s_h,a_h)\mid\mathcal H_h\right]
&\ge
\mathbb E\!\left[\widehat Q_h(s_h,a_h)\mid\mathcal H_h\right]
-\delta_{\mathrm{crit}}.
\end{align*}
Subtracting the second inequality from the first and using~\eqref{eq:finite_session_selector_condition} gives
\[
\mathbb E[\zeta_h\mid\mathcal H_h]
\le2\delta_{\mathrm{crit}}+\delta_{\mathrm{sel},h}.
\]
The assumed integrability permits taking expectations. Averaging over stages gives
\begin{equation}
\frac1H\sum_{h=1}^H\mathbb E[\zeta_h]
\le2\delta_{\mathrm{crit}}+\overline\delta_{\mathrm{sel},H}.
\label{eq:finite_session_selection_bound}
\end{equation}
Finally, the filtering and selection bounds control the two components of each gap in~\eqref{eq:episodic_performance_identity}. Substituting~\eqref{eq:finite_session_filtering_bound} and~\eqref{eq:finite_session_selection_bound}, using $\Delta_h+\zeta_h$, and multiplying by $H$ gives~\eqref{eq:finite_session_value_bound}. No step requires the policy, critic, or threshold to converge; the argument uses their stated stagewise information and error conditions.
\end{proof}

\subsection{Proof of Greedy Preservation and Conditional Recovery}
We first show that thresholding, capping, and fallback all preserve the same critic maximizer. We then use that preservation to express policy loss through pairwise critic error and the remaining selection gap.

\begin{proof}[Proof of Proposition~\ref{prop:greedy_preservation}]
Fix $t$ and $s$, and let $\widehat a_t(s)$ be the first maximizer of $\widehat Q_t(s,\cdot)$ under the fixed action order. Every score $g_t(s,a)$ is nonnegative, and $g_t(s,\widehat a_t(s))=0$. If $\tau_t<0$, the raw set is empty and the minimum score fallback is exactly $\widehat a_t(s)$. If $\tau_t\ge0$, the raw set contains every critic maximizer. Its first action in increasing score order is $\widehat a_t(s)$, so every cap with $K_t\ge1$ retains this action. In both cases,
\[
\widehat a_t(s)\in D_t(s),
\qquad
\max_{a\in D_t(s)}\widehat Q_t(s,a)
=\max_a\widehat Q_t(s,a).
\]
Thus the execution set preserves the global critic maximizer selected by the fixed action order. A greedy selector using that order consequently returns the same item with or without filtering.

Choose any $a^*(s)\in\argmax_aQ^*(s,a)$. Substituting $Q^*=\widehat Q_t-\xi_t$ into the optimal Bellman residual gives
\begin{align*}
b_{\pi_t}(s)
&=Q^*(s,a^*(s))-\sum_a\pi_t(a\mid s)Q^*(s,a)\\
&=\widehat Q_t(s,a^*(s))
-\sum_a\pi_t(a\mid s)\widehat Q_t(s,a)\\
&\quad+\sum_a\pi_t(a\mid s)
\bigl[\xi_t(s,a)-\xi_t(s,a^*(s))\bigr].
\end{align*}
For the first difference, replace $\widehat Q_t(s,a^*(s))$ by its maximum over all actions, which equals its maximum over $D_t(s)$ by the preceding argument. This upper bounds that difference by $\sigma_t(s)$. Every bracket in the final sum is at most $\beta_t$, and the policy weights are nonnegative and sum to one. Consequently,
\[
0\le b_{\pi_t}(s)\le\sigma_t(s)+\beta_t.
\]
The residual identity in the proof of Theorem~\ref{thm:value_decomposition} now gives
\begin{align*}
V^*(\mu_0)-V^{\pi_t}(\mu_0)
&=\frac1{1-\gamma}
\mathbb E_{s\sim d^{\pi_t}}[b_{\pi_t}(s)]\\
&\le\frac{\overline\sigma_t+\beta_t}{1-\gamma}.
\end{align*}
For every state and action pair, the triangle inequality gives
\[
|\xi_t(s,a)-\xi_t(s,b)|
\le|\xi_t(s,a)|+|\xi_t(s,b)|
\le2\|\widehat Q_t-Q^*\|_\infty.
\]
Maximizing shows that $\beta_t\le2\|\widehat Q_t-Q^*\|_\infty$.

Suppose a strictly suboptimal state and action pair exists. Finiteness makes the minimum of its positive gaps strictly positive, so $\Delta_{\min}>0$ is well defined. Let $\widehat a\in\argmax_a\widehat Q_t(s,a)$. Then
\begin{align*}
V^*(s)-Q^*(s,\widehat a)
&=\widehat Q_t(s,a^*(s))-\widehat Q_t(s,\widehat a)\\
&\quad+\xi_t(s,\widehat a)-\xi_t(s,a^*(s))\\
&\le\beta_t,
\end{align*}
because the first difference is nonpositive. If $\widehat a$ were suboptimal and $\beta_t<\Delta_{\min}$, this would imply
\[
\Delta_{\min}
\le V^*(s)-Q^*(s,\widehat a)
\le\beta_t<\Delta_{\min},
\]
a contradiction. This proves ranking consistency for every state. Since $D_t(s)$ contains a global critic maximizer, every critic maximizer within $D_t(s)$ is also a global critic maximizer. A policy supported on these actions is therefore supported on the true optimal actions. Lemma~\ref{lem:bellman} proves its optimality. If no strictly suboptimal pair exists, every policy is supported on optimal actions, and the same lemma applies directly.

Finally, the quantitative value convergence statement follows by taking limits in the value bound above. For exact greedy selection, $\beta_t\to0$ implies $\beta_t<\Delta_{\min}$ at every sufficiently large index when a positive gap exists. At those indices the policy is optimal. A unique optimal action at each state forces the action distribution to be the same point mass at all such indices. If several actions are optimal, support on those actions imposes no requirement that the distribution converge. None of these steps uses convergence of a parameter vector or of the threshold.
\end{proof}

\section{Experiment Details}

This appendix describes the comparison methods and the reported metrics. The metric definitions distinguish utility, catalog exposure, and retained set size so that the numerical comparisons can be interpreted separately from the calibration guarantee.

\subsection{Baseline Details}\label{app:baseline}

\begin{itemize}
\item \textbf{A2C}: an advantage actor--critic method that updates the policy using advantage-weighted policy gradients and trains the critic to estimate the state-value function~\citep{mnih2016asynchronous}.

\item \textbf{DDPG}: an off-policy deterministic actor--critic algorithm designed for continuous action spaces, which is applied to continuous action representations for slate recommendation~\citep{lillicrap2016continuous}.

\item \textbf{TD3}: an extension of DDPG that mitigates value overestimation and improves training stability through clipped double Q-learning, delayed policy updates, and target policy smoothing~\citep{fujimoto2018addressing}.

\item \textbf{HAC}: a hyper-actor--critic method that represents recommendation slates in a latent action space and decomposes slate generation into hyper-action inference and effect-action selection~\citep{liu2023exploration}.

\end{itemize}

\subsection{Evaluation Metrics}
\label{app:evaluation_metrics}

The metrics below describe complete simulated sessions. Let $N$ be the number of evaluated sessions, $T_n$ the length of session $n$, and $R_{n,t}$ its reward at request $t$. Let $L_{n,t}$ denote the item collection used by the evaluation code at that request. A baseline with fixed slate size has $|L_{n,t}|=M$, while the reported RLCP set size uses its retained execution collection. Catalog diversity and ILD require the evaluated collection to be defined consistently across methods.

\paragraph{Depth.}
We first measure how long the interaction continues. Depth is the mean number of recommendation steps before departure:
\[
\mathrm{Depth}
=
\frac{1}{N}\sum_{n=1}^{N} T_n.
\]
A larger value means that the simulated sessions contain more interactions; it does not specify how much reward those interactions produce.

\paragraph{Average Reward.}
To measure the reward per interaction while giving equal weight to each session, we first average within sessions and then across sessions:
\[
\mathrm{AvgReward}
=
\frac{1}{N}\sum_{n=1}^{N}
\frac{1}{T_n}
\sum_{t=1}^{T_n} R_{n,t}.
\]

\paragraph{Total Reward.}
In contrast, total reward records the cumulative reward of a session before averaging across sessions:
\[
\mathrm{TotalReward}
=
\frac{1}{N}\sum_{n=1}^{N}
\sum_{t=1}^{T_n} R_{n,t}.
\]

\paragraph{Diversity.}
To measure exposure across users, catalog diversity counts distinct items in the $B$ active requests of an evaluation batch and averages this count over $H_{\mathrm{eval}}$ evaluation steps, following KuaiSim:
\[
\mathrm{Diversity}
=
\frac{1}{H_{\mathrm{eval}}}
\sum_{h=1}^{H_{\mathrm{eval}}}
\left|
\bigcup_{j=1}^{B}L_{j,h}
\right|.
\]
The tables call this metric \emph{Diversity}. It counts exposed items and is distinct from the proxy retention probability controlled by calibration.

\paragraph{Intra List Diversity.}
Catalog diversity does not describe how similar the items within a request are. ILD measures that separate property through mean pairwise dissimilarity within each evaluated collection. Write $M_{n,t}=|L_{n,t}|$ and let $e_a$ be the embedding of item $a$. Then
\begin{align*}
\mathrm{ILD}
&=
\frac{1}{\sum_{n=1}^{N} T_n}
\sum_{n=1}^{N}
\sum_{t=1}^{T_n}
\frac{2}{M_{n,t}(M_{n,t}-1)}
\nonumber\\
&\quad\times
\sum_{1 \leq j < k \leq M_{n,t}}
d\!\left(
  e_{(L_{n,t})_j},
  e_{(L_{n,t})_k}
\right).
\end{align*}
Here \(d(\cdot,\cdot)\) is the item dissimilarity function and \(M_{n,t}\) is the number of items in the evaluated collection at request \(t\) of session \(n\). For \(M_{n,t}=1\), the corresponding diversity term is defined as zero.

\paragraph{Admissible Set Size.}
Finally, the average admissible set size records how many items remain after filtering, averaged over evaluation requests:
\[
\mathrm{SetSize}
=
\frac{1}{\sum_{n=1}^{N} T_n}
\sum_{n=1}^{N}
\sum_{t=1}^{T_n}
|L_{n,t}|.
\]
A smaller value indicates that fewer alternatives are retained. It is not a measure of total computation, since it excludes scoring the candidate pool and obtaining the proxy label.

\section{Full Results}
\label{app:full_results}
The following tables report the complete measurements underlying Section~\ref{sec:experiments}. Each configuration fixes the dataset, reward signal, patience parameters, and baseline slate size $M$, with the matched cap $K=M$ for RLCP. Comparisons therefore concern the same maximum set size within a configuration. Boldface identifies the highest depth and catalog diversity and the lowest mean set size, including ties. These markings compare point estimates and do not imply statistical significance.

\begin{table*}[h]
\centering
\small
\setlength{\tabcolsep}{1.45pt}
\caption{KuaiRand-Pure results under patience setting \((p,q)=(1,2)\) with \texttt{is\_click} as the reward signal for displayed slate sizes \(M\in\{3,4,5\}\). The final seven columns report behavior rates; set size is reported for RLCP variants when available.}
\label{tab:kuairand_pq1_slate}
\begin{tabular}{l*{13}{r}}
\toprule
Method & \textbf{Depth} & \shortstack{Avg.\\reward} & \shortstack{Total\\reward} & \textbf{Diversity} & ILD & \shortstack{\textbf{Set}\\\textbf{size}} & Click & Comment & Follow & Forward & Hate & Like & \shortstack{Long\\view} \\
\midrule
\multicolumn{14}{l}{\textbf{Slate size $M=3$}} \\
 DDPG & 18.80 & 0.62 & 13.46 & 25.00 & 0.99 & 3.00 & 0.62 & 0.61 & 0.54 & 0.54 & 0.55 & 0.58 & 0.64 \\
TD3 & 19.20 & 0.65 & 13.24 & 54.50 & 0.99 & 3.00 & 0.65 & 0.62 & 0.52 & 0.54 & 0.52 & 0.58 & 0.63 \\
 A2C & \textbf{19.60} & 0.66 & 13.32 & 3.00 & 0.99 & 3.00 & 0.66 & 0.61 & 0.58 & 0.56 & 0.54 & 0.60 & 0.62 \\
 HAC & 19.10 & 0.67 & 13.47 & 4.70 & 0.98 & 3.00 & 0.67 & 0.62 & 0.57 & 0.55 & 0.54 & 0.56 & 0.68 \\
RLCP-Single & 19.30 & 0.64 & 13.20 & \textbf{92.40} & 0.99 & \textbf{2.91} & 0.64 & 0.57 & 0.55 & 0.55 & 0.52 & 0.59 & 0.65 \\
RLCP & 19.40 & 0.61 & 12.77 & 28.50 & 1.00 & 3.00 & 0.61 & 0.58 & 0.56 & 0.55 & 0.54 & 0.57 & 0.59 \\
\midrule
\multicolumn{14}{l}{\textbf{Slate size $M=4$}} \\
DDPG & 19.60 & 0.66 & 13.55 & 19.50 & 0.99 & \textbf{4.00} & 0.66 & 0.61 & 0.58 & 0.58 & 0.52 & 0.58 & 0.65 \\
 TD3 & 19.70 & 0.61 & 13.49 & 84.40 & 0.99 & \textbf{4.00} & 0.61 & 0.61 & 0.55 & 0.52 & 0.54 & 0.56 & 0.61 \\
 A2C & \textbf{19.90} & 0.67 & 13.35 & 15.50 & 0.99 & \textbf{4.00} & 0.67 & 0.65 & 0.59 & 0.57 & 0.55 & 0.63 & 0.66 \\
 HAC & 19.60 & 0.67 & 13.62 & 6.60 & 0.99 & \textbf{4.00} & 0.67 & 0.59 & 0.58 & 0.54 & 0.50 & 0.61 & 0.65 \\
RLCP-Single & \textbf{19.90} & 0.64 & 13.52 & \textbf{126.40} & 0.99 & \textbf{4.00} & 0.64 & 0.61 & 0.59 & 0.55 & 0.54 & 0.60 & 0.64 \\
 RLCP & \textbf{19.90} & 0.59 & 13.10 & 50.50 & 1.00 & \textbf{4.00} & 0.59 & 0.58 & 0.52 & 0.51 & 0.52 & 0.56 & 0.59 \\
\midrule
\multicolumn{14}{l}{\textbf{Slate size $M=5$}} \\
 DDPG & 19.90 & 0.68 & 13.67 & 17.50 & 0.99 & 5.00 & 0.68 & 0.60 & 0.56 & 0.56 & 0.54 & 0.57 & 0.65 \\
TD3 & 19.90 & 0.64 & 13.67 & 47.20 & 0.99 & 5.00 & 0.64 & 0.58 & 0.54 & 0.53 & 0.52 & 0.58 & 0.63 \\
A2C & \textbf{20.00} & 0.65 & 13.35 & 24.80 & 0.99 & 5.00 & 0.65 & 0.61 & 0.58 & 0.54 & 0.53 & 0.60 & 0.63 \\
HAC & \textbf{20.00} & 0.66 & 13.99 & 10.00 & 0.99 & 5.00 & 0.66 & 0.60 & 0.58 & 0.55 & 0.53 & 0.59 & 0.65 \\
RLCP-Single & \textbf{20.00} & 0.66 & 13.93 & \textbf{154.10} & 0.99 & \textbf{4.91} & 0.66 & 0.56 & 0.53 & 0.53 & 0.55 & 0.55 & 0.64 \\
RLCP & \textbf{20.00} & 0.63 & 13.03 & 36.40 & 1.00 & 5.00 & 0.63 & 0.56 & 0.56 & 0.53 & 0.54 & 0.56 & 0.63 \\
\bottomrule
\end{tabular}%
\end{table*}

\begin{table*}[t]
\centering
\small
\setlength{\tabcolsep}{1.85pt}
\caption{KuaiRand-Pure results under patience setting \((p,q)=(0.4,2)\) with \texttt{is\_click} as the reward signal for displayed slate sizes \(M\in\{4,5\}\). The final seven columns report behavior rates; set size is reported for RLCP variants when available.}
\label{tab:kuairand_pq04_click_slate}
\begin{tabular}{l*{13}{r}}
\toprule
Method & \textbf{Depth} & \shortstack{Avg.\\reward} & \shortstack{Total\\reward} & \textbf{Diversity} & ILD & \shortstack{\textbf{Set}\\\textbf{size}} & Click & Comment & Follow & Forward & Hate & Like & \shortstack{Long\\view} \\
\midrule
\multicolumn{14}{l}{\textbf{Slate size $M=4$}} \\
DDPG & 46.80 & 0.70 & 32.83 & 14.40 & 0.99 & 4.00 & 0.70 & 0.62 & 0.59 & 0.57 & 0.52 & 0.60 & 0.68 \\
TD3 & \textbf{48.00} & 0.69 & 33.46 & 25.90 & 0.99 & 4.00 & 0.69 & 0.64 & 0.59 & 0.60 & 0.53 & 0.61 & 0.70 \\
A2C & 46.00 & 0.64 & 31.35 & 4.10 & 0.98 & 4.00 & 0.64 & 0.60 & 0.57 & 0.54 & 0.55 & 0.57 & 0.63 \\
HAC & \textbf{48.00} & 0.68 & 32.95 & 7.10 & 0.99 & 4.00 & 0.68 & 0.64 & 0.58 & 0.57 & 0.53 & 0.61 & 0.69 \\
RLCP-Single & 47.60 & 0.68 & 32.92 & \textbf{118.10} & 0.99 & \textbf{3.84} & 0.68 & 0.62 & 0.57 & 0.54 & 0.55 & 0.60 & 0.65 \\
RLCP & 47.20 &  0.62  &  31.91  & 41.60 &  1.00  & 3.91 &  0.63  &  0.57  &  0.58  &  0.55  &  0.54  &  0.59  &  0.62 \\
\midrule
\multicolumn{14}{l}{\textbf{Slate size $M=5$}} \\
DDPG & \textbf{48.00} & 0.68 & 33.63 & 20.00 & 0.98 & 5.00 & 0.68 & 0.62 & 0.60 & 0.58 & 0.52 & 0.61 & 0.66 \\
TD3 & \textbf{48.00} & 0.68 & 33.74 & 37.90 & 0.99 & 5.00 & 0.68 & 0.64 & 0.58 & 0.57 & 0.54 & 0.60 & 0.69 \\
A2C & \textbf{48.00} & 0.63 & 32.49 & 5.00 & 0.98 & 5.00 & 0.63 & 0.57 & 0.57 & 0.54 & 0.53 & 0.58 & 0.64 \\
HAC & \textbf{48.00} & 0.59 & 31.45 & 5.40 & 0.98 & 5.00 & 0.59 & 0.57 & 0.55 & 0.51 & 0.53 & 0.57 & 0.60 \\
RLCP-Single & \textbf{48.00} & 0.69 & 33.51 & \textbf{149.90} & 0.99 & \textbf{4.94} & 0.69 & 0.64 & 0.58 & 0.58 & 0.56 & 0.60 & 0.69\\
RLCP & \textbf{48.00} & 0.61 & 31.35 & 77.00 & 1.00 & 5.00 & 0.61 & 0.60 & 0.55 & 0.55 & 0.54 & 0.60 & 0.58 \\
\bottomrule
\end{tabular}%
\end{table*}

\begin{table*}[t]
\centering
\small
\setlength{\tabcolsep}{1.85pt}
\caption{KuaiRand-Pure results under patience setting \((p,q)=(1,2)\) with \texttt{is\_like} as the reward signal for displayed slate sizes \(M\in\{4,5,6\}\). The final seven columns report behavior rates; set size is reported for RLCP variants when available.}
\label{tab:kuairand_pq1_like_slate}
\begin{tabular}{l*{13}{r}}
\toprule
Method & \textbf{Depth} & \shortstack{Avg.\\reward} & \shortstack{Total\\reward} & \textbf{Diversity} & ILD & \shortstack{\textbf{Set}\\\textbf{size}} & Click & Comment & Follow & Forward & Hate & Like & \shortstack{Long\\view} \\
\midrule
\multicolumn{14}{l}{\textbf{Slate size $M=4$}} \\
DDPG & 19.80 & 0.61 & 12.79 & 24.00 & 0.99 & 4.00 & 0.66 & 0.60 & 0.58 & 0.55 & 0.54 & 0.61 & 0.65 \\
TD3 & 19.80 & 0.60 & 12.69 & 51.80 & 0.99 & 4.00 & 0.62 & 0.60 & 0.59 & 0.57 & 0.56 & 0.60 & 0.62 \\
A2C & 19.70 & 0.60 & 12.83 & 4.00 & 0.98 & 4.00 & 0.62 & 0.60 & 0.59 & 0.55 & 0.53 & 0.60 & 0.59 \\
HAC & \textbf{20.00} & 0.58 & 12.84 & 8.60 & 0.99 & 4.00 & 0.64 & 0.59 & 0.56 & 0.57 & 0.53 & 0.58 & 0.66 \\
RLCP-Single & 19.80 & 0.61 & 12.85 & \textbf{119.70} & 0.99 & \textbf{3.91} & 0.63 & 0.59 & 0.57 & 0.57 & 0.54 & 0.61 & 0.63 \\
RLCP & 19.70 & 0.59 & 12.27 & 58.60 & 1.00 & 4.00 & 0.64 & 0.59 & 0.58 & 0.54 & 0.54 & 0.59 & 0.64 \\
\midrule
\multicolumn{14}{l}{\textbf{Slate size $M=5$}} \\
DDPG & \textbf{20.00} & 0.58 & 12.88 & 21.80 & 0.99 & 5.00 & 0.62 & 0.57 & 0.56 & 0.54 & 0.54 & 0.58 & 0.58 \\
TD3 & 19.90 & 0.58 & 13.18 & 72.30 & 0.99 & 5.00 & 0.61 & 0.59 & 0.54 & 0.55 & 0.52 & 0.58 & 0.63 \\
A2C & 19.80 & 0.59 & 12.87 & 5.00 & 0.98 & 5.00 & 0.62 & 0.60 & 0.57 & 0.54 & 0.55 & 0.59 & 0.62 \\
HAC & \textbf{20.00} & 0.57 & 13.06 & 10.20 & 0.99 & 5.00 & 0.65 & 0.59 & 0.58 & 0.56 & 0.54 & 0.57 & 0.63 \\
RLCP-Single & \textbf{20.00} & 0.58 & 12.83 & \textbf{149.60} & 0.99 & \textbf{4.78} & 0.64 & 0.60 & 0.57 & 0.53 & 0.52 & 0.58 & 0.65 \\
RLCP & \textbf{20.00} & 0.56 & 11.98 & 64.40 & 1.00 & 5.00 & 0.62 & 0.57 & 0.57 & 0.53 & 0.54 & 0.56 & 0.62 \\
\midrule
\multicolumn{14}{l}{\textbf{Slate size $M=6$}} \\
DDPG & \textbf{20.00} & 0.57 & 12.62 & 25.50 & 0.99 & 6.00 & 0.59 & 0.59 & 0.55 & 0.55 & 0.53 & 0.57 & 0.57 \\
TD3 & \textbf{20.00} & 0.59 & 12.51 & 60.70 & 0.99 & 6.00 & 0.65 & 0.61 & 0.55 & 0.55 & 0.55 & 0.59 & 0.63 \\
A2C & \textbf{20.00} & 0.62 & 12.74 & 6.00 & 0.99 & 6.00 & 0.66 & 0.62 & 0.58 & 0.54 & 0.53 & 0.62 & 0.64 \\
HAC & \textbf{20.00} & 0.62 & 12.76 & 9.10 & 0.99 & 6.00 & 0.63 & 0.59 & 0.57 & 0.56 & 0.55 & 0.62 & 0.61 \\
RLCP-Single & 19.90 &  0.60 & 12.83 & \textbf{155.80} & 0.99 & \textbf{4.47} & 0.66 & 0.61 & 0.57 & 0.55 & 0.54 & 0.60 & 0.64 \\
RLCP & \textbf{20.00} & 0.58 & 11.46 & 100.20 & 1.00 & 6.00 & 0.60 & 0.56 & 0.55 & 0.53 & 0.51 & 0.58 & 0.59\\
\bottomrule
\end{tabular}%
\end{table*}

\begin{table*}[t]
\centering
\small
\setlength{\tabcolsep}{1.85pt}
\caption{KuaiRand-Pure results under patience setting \((p,q)=(0.4,2)\) with \texttt{is\_like} as the reward signal for displayed slate sizes \(M\in\{5,6\}\). The final seven columns report behavior rates; set size is reported for RLCP variants when available.}
\label{tab:kuairand_pq04_like_slate}
\begin{tabular}{l*{13}{r}}
\toprule
Method & \textbf{Depth} & \shortstack{Avg.\\reward} & \shortstack{Total\\reward} & \textbf{Diversity} & ILD & \shortstack{\textbf{Set}\\\textbf{size}} & Click & Comment & Follow & Forward & Hate & Like & \shortstack{Long\\view} \\
\midrule
\multicolumn{14}{l}{\textbf{Slate size $M=5$}} \\
DDPG & 47.60 & 0.57 & 31.62 & 35.20 & 0.99 & 5.00 & 0.65 & 0.60 & 0.57 & 0.55 & 0.57 & 0.57 & 0.67 \\
TD3 & \textbf{48.00} & 0.60 & 31.93 & 53.90 & 0.99 & 5.00 & 0.69 & 0.64 & 0.57 & 0.58 & 0.56 & 0.60 & 0.71 \\
A2C & 45.60 & 0.60 & 30.79 & 5.00 & 0.98 & 5.00 & 0.64 & 0.62 & 0.58 & 0.56 & 0.55 & 0.60 & 0.66 \\
HAC & 47.20 & 0.61 & 29.78 & 5.10 & 0.98 & 5.00 & 0.63 & 0.59 & 0.59 & 0.52 & 0.54 & 0.61 & 0.63 \\
RLCP-Single & \textbf{48.00} & 0.62 & 31.09 & \textbf{151.70} & 0.99 & \textbf{4.84} & 0.70 & 0.63 & 0.61 & 0.57 & 0.55 & 0.62 & 0.70 \\
RLCP & 47.70 & 0.56 & 29.20 & 72.20 & 1.00 & 4.97 & 0.60 & 0.58 & 0.56 & 0.53 & 0.53 & 0.57 & 0.58 \\
\midrule
\multicolumn{14}{l}{\textbf{Slate size $M=6$}} \\
DDPG & \textbf{48.00} & 0.61 & 30.86 & 58.30 & 0.99 & 6.00 & 0.67 & 0.62 & 0.59 & 0.56 & 0.52 & 0.61 & 0.67 \\
TD3 & 46.80 & 0.60 & 31.75 & 124.70 & 0.99 & 6.00 & 0.71 & 0.64 & 0.60 & 0.58 & 0.52 & 0.60 & 0.70 \\
A2C & \textbf{48.00} & 0.63 & 30.11 & 6.00 & 0.98 & 6.00 & 0.70 & 0.63 & 0.59 & 0.58 & 0.56 & 0.63 & 0.69 \\
HAC & 47.90 & 0.59 & 30.54 & 7.50 & 0.98 & 6.00 & 0.65 & 0.60 & 0.58 & 0.53 & 0.52 & 0.59 & 0.64 \\
RLCP-Single & \textbf{48.00} & 0.62 & 31.22 & \textbf{184.20} & 0.99 & \textbf{5.72} & 0.70 & 0.65 & 0.61 & 0.58 & 0.53 & 0.62 & 0.69 \\
RLCP & \textbf{48.00} & 0.59 & 28.52 & 85.50 & 1.00 & 6.00 & 0.61 & 0.55 & 0.55 & 0.54 & 0.53 & 0.60 & 0.57 \\
\bottomrule
\end{tabular}%
\end{table*}

\begin{table*}[t]
\centering
\small
\setlength{\tabcolsep}{3pt}
\caption{ML-1M results under patience setting \((p,q)=(1,2)\) with \texttt{is\_click} as the reward signal for displayed slate sizes \(M\in\{4,5\}\). The final three columns report behavior rates; set size is reported for RLCP variants when available.}
\label{tab:ml1m_pq1_slate}
\begin{tabular}{l*{9}{r}}
\toprule
Method & \textbf{Depth} & \shortstack{Avg.\\reward} & \shortstack{Total\\reward} & \textbf{Diversity} & ILD & \shortstack{\textbf{Set}\\\textbf{size}} & Click & Like & Star \\
\midrule
\multicolumn{10}{l}{\textbf{Slate size $M=4$}} \\
DDPG & 19.90 & 0.69 & 13.95 & 23.90 & 0.99 & \textbf{4.00} & 0.69 & 0.69 & 0.70 \\
TD3 & 19.90 & 0.73 & 14.90 & 15.90 & 0.99 & \textbf{4.00} & 0.73 & 0.72 & 0.70 \\
A2C & 19.20 & 0.58 & 11.97 & 46.30 & 0.99 & \textbf{4.00} & 0.58 & 0.61 & 0.62 \\
HAC & \textbf{20.00} & 0.68 & 14.30 & 4.80 & 0.99 & \textbf{4.00} & 0.68 & 0.66 & 0.66 \\
RLCP-Single & 19.80 & 0.58 & 11.85 & \textbf{119.40} & 0.99 & \textbf{4.00} & 0.58 & 0.58 & 0.58 \\
RLCP & \textbf{20.00} & 0.61 & 12.85 & 41.80 & 1.00 & \textbf{4.00} & 0.61 & 0.57 & 0.55\\
\midrule
\multicolumn{10}{l}{\textbf{Slate size $M=5$}} \\
DDPG & 19.90 & 0.69 & 13.82 & 35.20 & 0.99 & 5.00 & 0.69 & 0.68 & 0.69 \\
TD3 & \textbf{20.00} & 0.73 & 14.87 & 16.30 & 0.99 & 5.00 & 0.73 & 0.72 & 0.75 \\
A2C & 19.90 & 0.58 & 11.86 & 49.50 & 0.99 & 5.00 & 0.58 & 0.61 & 0.61 \\
HAC & \textbf{20.00} & 0.63 & 12.80 & 6.70 & 0.99 & 5.00 & 0.63 & 0.61 & 0.63 \\
RLCP-Single & \textbf{20.00} & 0.57 & 12.41 & 17.30 & 0.99 & 5.00 & 0.57 & 0.61 & 0.62 \\
RLCP & 19.80 & 0.60 & 13.71 & \textbf{54.90} & 1.00 & \textbf{4.53} & 0.64 & 0.59 & 0.57 \\
\bottomrule
\end{tabular}%
\end{table*}

\begin{table*}[t]
\centering
\small
\setlength{\tabcolsep}{3pt}
\caption{ML-1M results under patience setting \((p,q)=(0.4,2)\) with \texttt{is\_click} as the reward signal for displayed slate sizes \(M\in\{4,5\}\). The final three columns report behavior rates; set size is reported for RLCP variants when available.}
\label{tab:ml1m_eval_report_slate}
\begin{tabular}{l*{9}{r}}
\toprule
Method & \textbf{Depth} & \shortstack{Avg.\\reward} & \shortstack{Total\\reward} & \textbf{Diversity} & ILD & \shortstack{\textbf{Set}\\\textbf{size}} & Click & Like & Star \\
\midrule
\multicolumn{10}{l}{\textbf{Slate size $M=4$}} \\
DDPG & \textbf{46.80} & 0.69 & 33.76 & 12.90 & 0.99 & 4.00 & 0.69 & 0.71 & 0.73 \\
TD3 & \textbf{46.80} & 0.71 & 35.13 & 14.30 & 0.99 & 4.00 & 0.71 & 0.71 & 0.73 \\
A2C & 44.00 & 0.59 & 26.76 & 31.20 & 0.99 & 4.00 & 0.59 & 0.64 & 0.64 \\
HAC & 45.20 & 0.61 & 29.02 & 9.70 & 0.99 & 4.00 & 0.61 & 0.64 & 0.65 \\
RLCP-Single & 45.60 & 0.57 & 26.56 & \textbf{119.30} & 0.99 & 3.97 & 0.57 & 0.59 & 0.63\\
RLCP & \textbf{46.80} & 0.57 & 27.95 & 62.00 & 1.00 & \textbf{3.78} & 0.59 & 0.60 & 0.60 \\
\midrule
\multicolumn{10}{l}{\textbf{Slate size $M=5$}} \\
DDPG & 47.20 & 0.69 & 34.00 & 15.80 & 0.99 & 5.00 & 0.69 & 0.69 & 0.72 \\
TD3 & \textbf{48.00} & 0.72 & 34.98 & 18.20 & 0.99 & 5.00 & 0.72 & 0.71 & 0.71 \\
A2C & \textbf{48.00} & 0.59 & 28.88 & 7.80 & 0.99 & 5.00 & 0.59 & 0.66 & 0.68 \\
HAC & \textbf{48.00} & 0.65 & 32.25 & 5.30 & 0.99 & 5.00 & 0.65 & 0.67 & 0.70 \\
RLCP-Single & 47.60 & 0.59 & 28.12 & \textbf{94.80} & 0.99 & \textbf{4.97} & 0.59 & 0.62 & 0.63 \\
RLCP & \textbf{48.00} & 0.63 & 30.64 & 45.10 & 1.00 & \textbf{4.97} & 0.63 & 0.63 & 0.62 \\
\bottomrule
\end{tabular}%
\end{table*}

\begin{table*}[t]
\centering
\small
\setlength{\tabcolsep}{3pt}
\caption{ML-1M results under patience setting \((p,q)=(1,2)\) with \texttt{is\_like} as the reward signal for displayed slate sizes \(M\in\{4,5,6\}\). The final three columns report behavior rates; set size is reported for RLCP variants when available.}
\label{tab:ml1m_pq1_like_slate}
\begin{tabular}{l*{9}{r}}
\toprule
Method & \textbf{Depth} & \shortstack{Avg.\\reward} & \shortstack{Total\\reward} & \textbf{Diversity} & ILD & \shortstack{\textbf{Set}\\\textbf{size}} & Click & Like & Star \\
\midrule
\multicolumn{10}{l}{\textbf{Slate size $M=4$}} \\
DDPG & 19.80 & 0.65 & 13.21 & 32.80 & 0.99 & \textbf{4.00} & 0.63 & 0.65 & 0.67 \\
TD3 & 19.80 & 0.66 & 13.71 & 20.80 & 0.99 & \textbf{4.00} & 0.65 & 0.66 & 0.66 \\
A2C & 19.60 & 0.60 & 12.83 & 4.10 & 0.99 & \textbf{4.00} & 0.57 & 0.60 & 0.60 \\
HAC & \textbf{19.90} & 0.64 & 13.31 & 13.90 & 0.99 & \textbf{4.00} & 0.67 & 0.64 & 0.63 \\
RLCP-Single & 19.80 & 0.63 & 12.75 & 7.70 & 0.99 & \textbf{4.00} & 0.58 & 0.63 & 0.62\\
RLCP & \textbf{19.90} & 0.61 & 12.49 & \textbf{74.10} & 1.00 & \textbf{4.00} & 0.62 & 0.61 & 0.60 \\
\midrule
\multicolumn{10}{l}{\textbf{Slate size $M=5$}} \\
DDPG & 19.90 & 0.62 & 13.09 & 52.00 & 0.99 & \textbf{5.00} & 0.64 & 0.62 & 0.63 \\
TD3 & \textbf{20.00} & 0.67 & 14.06 & 24.50 & 0.99 & \textbf{5.00} & 0.67 & 0.67 & 0.67 \\
A2C & 19.80 & 0.67 & 13.05 & 5.70 & 0.99 & \textbf{5.00} & 0.60 & 0.67 & 0.65 \\
HAC & 19.90 & 0.67 & 13.41 & 11.40 & 0.99 & \textbf{5.00} & 0.63 & 0.67 & 0.66 \\
RLCP-Single & \textbf{20.00} & 0.63 & 13.03 & 11.50 & 0.99 & \textbf{5.00} & 0.60 & 0.63 & 0.62\\
RLCP & \textbf{20.00} & 0.62 & 12.34 & \textbf{64.60} & 1.00 & \textbf{5.00} & 0.63 & 0.62 & 0.60 \\
\midrule
\multicolumn{10}{l}{\textbf{Slate size $M=6$}} \\
DDPG & \textbf{20.00} & 0.66 & 13.46 & 59.10 & 0.99 & 6.00 & 0.63 & 0.66 & 0.67 \\
TD3 & \textbf{20.00} & 0.69 & 13.73 & 26.20 & 0.99 & 6.00 & 0.65 & 0.69 & 0.70 \\
A2C & 19.80 & 0.65 & 13.19 & 9.60 & 0.99 & 6.00 & 0.62 & 0.65 & 0.69 \\
HAC & \textbf{20.00} & 0.63 & 13.34 & 28.60 & 0.99 & 6.00 & 0.61 & 0.63 & 0.64 \\
RLCP-Single & \textbf{20.00} & 0.61 & 12.67 & \textbf{173.70} & 0.99 & \textbf{5.94} & 0.56 & 0.61 & 0.61\\
RLCP & \textbf{20.00} & 0.58 & 12.00 & 69.20 & 1.00 & 6.00 & 0.61 & 0.58 & 0.56 \\
\bottomrule
\end{tabular}%
\end{table*}

\begin{table*}[t]
\centering
\small
\setlength{\tabcolsep}{3pt}
\caption{ML-1M results under patience setting \((p,q)=(0.4,2)\) with \texttt{is\_like} as the reward signal for displayed slate sizes \(M\in\{4,5\}\). The final three columns report behavior rates; set size is reported for RLCP variants when available.}
\label{tab:ml1m_pq04_like_slate}
\begin{tabular}{l*{9}{r}}
\toprule
Method & \textbf{Depth} & \shortstack{Avg.\\reward} & \shortstack{Total\\reward} & \textbf{Diversity} & ILD & \shortstack{\textbf{Set}\\\textbf{size}} & Click & Like & Star \\
\midrule
\multicolumn{10}{l}{\textbf{Slate size $M=4$}} \\
DDPG & \textbf{47.20} & 0.73 & 33.86 & 20.10 & 0.99 & 4.00 & 0.68 & 0.73 & 0.70 \\
TD3 & \textbf{47.20} & 0.70 & 34.57 & 20.60 & 0.99 & 4.00 & 0.69 & 0.70 & 0.70 \\
A2C & 41.20 & 0.64 & 29.45 & 22.30 & 0.99 & 4.00 & 0.60 & 0.64 & 0.63 \\
HAC & 44.20 & 0.71 & 31.59 & 6.80 & 0.99 & 4.00 & 0.68 & 0.71 & 0.69 \\
RLCP-Single & 46.80 & 0.63 & 29.46 & 12.40 & 0.99 & 4.00 & 0.61 & 0.63 & 0.63\\
RLCP & 45.60 & 0.57 & 28.28 & \textbf{70.60} & 1.00 & \textbf{3.84} & 0.58 & 0.58 & 0.60 \\
\midrule
\multicolumn{10}{l}{\textbf{Slate size $M=5$}} \\
DDPG & \textbf{48.00} & 0.70 & 34.83 & 22.40 & 0.99 & \textbf{5.00} & 0.71 & 0.70 & 0.70 \\
TD3 & \textbf{48.00} & 0.73 & 34.48 & 11.60 & 0.99 & \textbf{5.00} & 0.71 & 0.73 & 0.72 \\
A2C & 47.70 & 0.64 & 31.75 & 7.40 & 0.99 & \textbf{5.00} & 0.62 & 0.64 & 0.65 \\
HAC & 47.20 & 0.65 & 32.45 & 25.10 & 0.99 & \textbf{5.00} & 0.65 & 0.65 & 0.68 \\
RLCP-Single & \textbf{48.00} & 0.62 & 31.27 & 8.50 & 0.99 & \textbf{5.00} & 0.63 & 0.62 & 0.62\\
RLCP & \textbf{48.00} & 0.57 & 27.90 & \textbf{82.10} & 1.00 & \textbf{5.00} & 0.58 & 0.57 & 0.59 \\
\bottomrule
\end{tabular}%
\end{table*}

\end{document}